\documentclass{article} % For LaTeX2e
\usepackage{iclr2020_conference,times}

\usepackage{amsmath,amsfonts,bm}

\def\eqref#1{equation~\ref{#1}}
\def\1{\bm{1}}

\def\va{{\bm{a}}}

\def\vs{{\bm{s}}}

\DeclareMathAlphabet{\mathsfit}{\encodingdefault}{\sfdefault}{m}{sl}
\SetMathAlphabet{\mathsfit}{bold}{\encodingdefault}{\sfdefault}{bx}{n}

\usepackage{hyperref}
\usepackage{url}
\usepackage[pdftex]{graphicx}

\newcommand{\ie}{\textit{i}.\textit{e}.}
\newcommand{\etal}{\textit{et al}.}

\title{Influence-Based Multi-Agent Exploration}
\usepackage{hyperref}
\makeatletter
\newcommand{\printfnsymbol}[1]{%
  \textsuperscript{\@fnsymbol{#1}}%
}
\makeatother
\author{Tonghan Wang\thanks{Equal Contribution.}, Jianhao Wang\printfnsymbol{1}, Yi Wu \& Chongjie Zhang \\
Institute for Interdisciplinary Information Sciences\\
Tsinghua University\\
Beijing, China \\
\texttt{wangth18@mails.tsinghua.edu.cn, wjh720.eric@gmail.com} \\
\texttt{jxwuyi@gmail.com, chongjie@tsinghua.edu.cn}
}

\usepackage{amsthm}

\newtheorem{lemma}{Lemma}

\usepackage{thmtools}
\usepackage{thm-restate}
\usepackage{cleveref}

\allowdisplaybreaks

\usepackage{color}
\usepackage{mathtools}

\usepackage{subfigure}
\usepackage{multirow}
\usepackage{makecell}
\usepackage{array, booktabs, arydshln, xcolor}

\newcommand{\shortn}{\textup{\texttt{-}}}
\newcommand{\shorte}{\textup{\texttt{=}}}

\newcolumntype{L}{>{$}l<{$}}
\newcolumntype{C}{>{$}c<{$}}
\newcolumntype{R}{>{$}r<{$}}

\iclrfinalcopy % Uncomment for camera-ready version, but NOT for submission.
\begin{document}

\maketitle

\begin{abstract}
Intrinsically motivated reinforcement learning aims to address the exploration challenge for sparse-reward tasks. However, the study of exploration methods in transition-dependent multi-agent settings is largely absent from the literature. We aim to take a step towards solving this problem. We present two exploration methods: exploration via information-theoretic influence (EITI) and exploration via decision-theoretic influence (EDTI), by exploiting the role of interaction in coordinated behaviors of agents. EITI uses mutual information to capture influence  transition dynamics. EDTI uses a novel intrinsic reward, called Value of Interaction (VoI), to characterize and quantify the influence of one agent's behavior on expected returns of other agents. By optimizing EITI or EDTI objective as a regularizer, agents are encouraged to coordinate their exploration and learn policies to optimize team performance. We show how to optimize these regularizers so that they can be easily integrated with policy gradient reinforcement learning. The resulting update rule draws a connection between coordinated exploration and intrinsic reward distribution. Finally, we empirically demonstrate the significant strength of our method in a variety of multi-agent scenarios.
\end{abstract}

\section{Introduction}
Reinforcement learning algorithms aim to learn a policy that maximizes the accumulative reward from an environment. Many advances of deep reinforcement learning rely on a dense shaped reward function, such as distance to the goal~\citep{mirowski2016learning,wu2018building}, scores in games~\citep{mnih2015human} or expert-designed rewards~\citep{wu2016training,OpenAI_dota}, while tend to struggle in many real-world scenarios with sparse rewards. Therefore, many recent works propose to introduce additional intrinsic incentives to boost exploration, including pseudo-counts~\citep{bellemare2016unifying,tang2017exploration,ostrovski2017count}, model-learning improvements~\citep{burda2018large,pathak2017curiosity,burda2018exploration}, and information gain~\citep{florensa2017stochastic,gupta2018meta,kim2019emi}. These works result in significant progress in many challenging tasks such as Montezuma Revenge~\citep{burda2018exploration}, robotic manipulation~\citep{pathak2018zero,riedmiller2018learning}, and Super Mario games~\citep{burda2018large,pathak2017curiosity}.

Notably, most of the existing breakthroughs on sparse-reward environments have been focusing on single-agent scenarios and leave the exploration problem largely unstudied for multi-agent settings -- it is common in real-world applications that multiple agents are required to solve a task in a coordinated fashion~\citep{cao2012overview,nowe2012game,zhang2011coordinated}. This problem has recently attracted attention and several exploration strategies have been proposed for many transition-independent cooperative multi-agent settings~\citep{dimakopoulou2018coordinated, dimakopoulou2018scalable, bargiacchi2018learning, iqbal2019coordinated}. Nevertheless, how to explore effectively in more general scenarios with complex reward and transition dependency among cooperative agents remains an open research problem.

This paper aims to take a step towards this goal. Our basic idea is to coordinate agents' exploration by taking into account their interactions during their learning processes. Configurations where interaction happens (interaction points) lie at critical junctions in the state-action space, through these critical configurations can transit to potentially important under-explored regions. To exploit this idea, we propose exploration strategies where agents start with decentralized exploration driven by their individual curiosity, and are also encouraged to visit interaction points to influence the exploration processes of other agents and help them get more extrinsic and intrinsic rewards. Based on how to quantify influence among agents, we propose two exploration methods. \emph{Exploration via information-theoretic influence} (EITI) uses mutual information (MI) to capture the interdependence between the transition dynamics of agents. \emph{Exploration via decision-theoretic influence} (EDTI) goes further and uses a novel measure called \emph{value of interaction} (VoI) to disentangle the effect of one agent's state-action pair on the expected (intrinsic) value of other agents. By optimizing MI or VoI as a regularizer to the value function, agents are encouraged to explore state-action pairs where they can exert influences on other agents for learning sophisticated multi-agent cooperation strategies. 

To efficiently optimize MI and VoI, we propose augmented policy gradient formulations so that the gradients can be estimated purely from trajectories. The resulting update rule draws a connection between coordinated exploration and the distribution of individual intrinsic rewards among team members, which further explains why our methods are able to facilitate multi-agent exploration.

We demonstrate the effectiveness of our methods on a variety of sparse-reward cooperative multi-agent tasks. Empirical results show that both EITI and EDTI allow for the discovery of influential states and EDTI further filter out interactions that have no effects on the performance. Our results also imply that these influential states are implicitly discovered as subgoals in search space that guide and coordinate exploration. The video of experiments is available at \url{https://sites.google.com/view/influence-based-mae/}.
% In multi-agent coordination tasks, directly extending those methods from the single-agent setting can be insufficient. Centralized exploration results in an exponentially large action space while decentralized exploration is lack of coordination...

\section{Related Works}\label{sec:related_works}
Single-agent exploration achieves conspicuous success recently. Provably efficient methods are proposed, such as upper confidence bound (UCB)~\citep{jaksch2010near, azar2017minimax, jin2018is} and posterior sampling for reinforcement learning (PSRL)~\citep{strens2000bayesian, osband2013more, osband2016lower, agrawal2017optimistic}. Given that these methods scale poorly to large or continuous settings, in another line of research, intrinsic rewards are used to guide the agent to explore what makes it curious~\citep{schmidhuber1991possibility, chentanez2005intrinsically, oudeyer2007intrinsic, barto2013intrinsic, bellemare2016unifying, pathak2017curiosity, ostrovski2017count}, and have obtained empirically solid results~\citep{burda2018large, burda2018exploration, kim2019emi}. \citet{houthooft2016vime} and~\citet{barron2018information} use variational information maximizing method to facilitate exploration.

Compared to the success in single-agent domain, progress in multi-agent exploration is less vigorous. \citet{dimakopoulou2018coordinated} and~\citet{dimakopoulou2018scalable} study coverage problems in the concurrent reinforcement learning, \citet{bargiacchi2018learning} propose an exploration method in repeated single-stage problems, and~\citet{iqbal2019coordinated} study several methods to combining decentralized curiosity. All these works focus on transition-independent settings. Another Bayesian exploration approach has been proposed for learning in stateless repeated games~\citep{Chalkiadakis:2003:CMR:860575.860689}. In contrast, this paper focuses on general multi-agent sequential decision making under uncertainty and complex reward and transition dependencies.

Some previous work has also studied intrinsic rewards in multi-agent reinforcement learning. The work that is most related to this paper is \citet{jaques2018intrinsic}, which uses social motivation to help multi-agent reinforcement learning where the influence of one agent on other agents' decision-making process is modelled. The main difference is that our method considers mainly the influence on other agent's transition function and rewarding structure and how to exploit these interactions to facilitate exploration. Hughes \etal~use inequality aversion reward to help intertemporal social dilemmas~\citep{hughes2018inequity}. \citet{strouse2018learning} use mutual information between goal and states or actions as an intrinsic reward to train the agent to share or hide their intentions.

Our method also draws inspiration from applications of mutual information in reinforcement learning. Mutual information is a critical statistical quantity and has been used in single-agent intrinsically motivated exploration~\citep{rubin2012trading, still2012information, salge2014changing, mohamed2015variational, kim2019emi} and multi-agent intention sharing and hiding~\citep{strouse2018learning}.

\section{Settings}

In our work, we consider a fully cooperative multi-agent task that can be modelled by a factored multi-agent MDP $G = \langle I, S, A, T, r, h\rangle$, where $I\equiv\{1,2,...,N\}$ is the finite set of agents, $S\equiv \times_{i\in I} S_i$ is the finite set of joint states and $S_i$ is the state set of agent $i$. At each timestep, each agent selects an action $a_i\in A_i$ at state $\vs$, forming a joint action $\va \in A \equiv \times_{i\in I} A_i,$ resulting in a shared extrinsic reward $r(\vs, \va)$ for each agent and the next state $\vs'$ according to the transition function $T(\vs' | \vs, \va)$.

The objective of the task is that each agent learns a policy $\pi_i(a_i|s_i)$, jointly maximizing team performance. The joint policy $\bm{\pi}\shorte\langle\pi_1, \dots, \pi_N\rangle$ induces an action-value function, $Q^{ext, \bm{\pi}}(\vs, \va) \shorte$ $\mathbb{E}_\tau[\sum_{t=0}^h r^t|\vs^0\shorte\vs, \va^0 \shorte \va, \bm{\pi}]$, and a value function $V^{ext, \bm{\pi}}(\vs)\shorte\max_{\va} Q^{ext, \bm{\pi}}(\vs, \va)$, where $\tau$ is the episode trajectory and $h$ is the horizon.

We adopt a centralized training and decentralized execution paradigm, which has been widely used in multi-agent deep reinforcement learning~\citep{foerster2016learning, lowe2017multi, foerster2018counterfactual, rashid2018qmix}. During training, agents are granted access to the states, actions, (intrinsic) rewards, and value functions of other agents, while decentralized execution only requires individual states.

\section{Influence-Based Coordinated Multi-Agent Exploration}

Efficient exploration is critical for reinforcement learning, particularly in sparse-reward tasks. Intrinsic motivation~\citep{oudeyer2009intrinsic} is a crucial mechanism for behaviour learning since it provides the driver of exploration. Therefore, to trade off exploration and exploitation, it is common for an RL agent to maximize an objective of the expected extrinsic reward augmented by the expected intrinsic reward. Curiosity is one of the extensively-studied intrinsic rewards to encourage an agent to explore according to its uncertainty about the environment, which can be measured by model prediction error~\citep{burda2018large,pathak2017curiosity,burda2018exploration} or state visitation count~\citep{bellemare2016unifying,tang2017exploration,ostrovski2017count}.

While such an intrinsic motivation as curiosity drives effective individual exploration, it is often not sufficient enough for learning in collaborative multi-agent settings, because it does not take into account agent interactions. To encourage interactions, we propose an influence value aims to quantify one agent's influence on the exploration processes of other agents. Maximizing this value will encourage agents to visit interaction points more often through which the agent team can reach configurations that are rarely visited by decentralized exploration. In next sections, we will provide two ways to formulate the influence value with such properties, leading to two exploration strategies. 
% Whether to mention here that these through points can get more intrinsic or extrinsic rewards?
% Define I is different of VoI. I is not concrete. Thus properties of I are just "hopes", we should say we can do that.
% IS 'with such properties' a GOOD writing?

Thus, for each agent $i$, our overall optimization objective is:
\begin{equation}
    \begin{aligned}
    J_{\theta_i}[\pi_{i} | \pi_{-i}, p_0] \equiv V^{ext, \bm{\pi}}(\vs_0) + V_i^{int, \bm{\pi}}(\vs_0) + \beta \cdot  I_{-i|i}^{\bm{\pi}},
    \end{aligned}
\end{equation}
where $p_0(\vs_0)$ is the initial state distribution, $\pi_{-i}$ is the joint policy excluding that of agent $i$, and $V_i^{int, \bm{\pi}}(\bm{s})$ is the intrinsic value function of agent $i$, $I_{-i|i}^{\bm{\pi}}$ is the influence value, $\beta>0$ is a weighting term. In this paper, we use the following notations:
\begin{eqnarray}
& &\tilde{r}_i(\vs, \va) =r(\vs, \va)+u_i(s_i, a_i), \\
& &V^{\bm{\pi}}_i(\vs) = V^{ext, \bm{\pi}}(\vs) + V_i^{int, \bm{\pi}}(\vs),\\
& &Q^{\bm{\pi}}_i(\vs,\va) = \tilde{r}_i(\vs, \va) + \sum_{\vs'}T(\vs'|\vs,\va)V_i^{\bm{\pi}}(\vs'),
\end{eqnarray}

where $u_i(s_i, a_i)$ is a curiosity-derived intrinsic reward, $\tilde{r}_i(\vs, \va)$ is a sum of intrinsic and extrinsic rewards, $V^{\bm{\pi}}_i(\vs)$ and $Q^{\bm{\pi}}_i(\vs,\va)$ here contain both intrinsic and extrinsic rewards.

% To address this challenge, we propose an additional social intrinsic reward, called \emph{Value of Interaction} (VoI), to  VoI is defined to characterize and quantify such influence of one agent's behaviour on other agent's learning processes, which will be discussed in detail in the next section.
% One straightforward method is to \emph{explore via information-theoretic influence} (EITI), e.g., using

\subsection{Exploration via Information-Theoretic Influence}\label{sec:mi}
One critical problem in our learning framework presented above is to define the influence value $I$. For simplicity, we start with a two-agent case. The first method we propose is to use mutual information between agents' trajectories to measure one agent's influence on other agents' learning processes. Such mutual information can be defined as information gain of one agent's state transition given the other's state and action. Without loss of generality, we define it from the perspective of agent 1:
\begin{equation}
    \begin{aligned}
    MI_{2|1}^{\bm{\pi}}(S_2';S_1,A_1|S_2,A_2) =& \sum_{\vs,\va,s_2' \in(S, A, S_2)} p^{\bm{\pi}}(\vs,\va,s_2') \left[\log p^{\bm{\pi}}(s_2'|\vs,\va) - \log p^{\bm{\pi}}(s_2'|s_2,a_2)\right],
    \end{aligned}
\end{equation}
where $\vs=(s_1,s_2)$ is the joint state, $\va=(a_1,a_2)$ is the joint action, and $S_i$ and $A_i$ are the random variables of state and action of agent $i$ subject to the distribution induced by the joint policy $\bm{\pi}$. So we define $I^{\bm{\pi}}_{2|1}$ as $MI_{2|1}^{\bm{\pi}}(S_2';S_1,A_1|S_2,A_2)$ that captures transition interactions between agents. Optimizing this objective encourages agent $1$ to visited critical points where it can influence the transition probability of agent 2. We call such an exploration method \emph{exploration via information-theoretic influence} (EITI).

Optimizing $MI_{2|1}^{\bm{\pi}}$ with respect to the policy parameters $\theta_1$ of agent 1 is a little bit challenging, because it is an expectation with respect to a distribution that depends on $\theta_1$. The gradient consists of two terms:
\begin{equation}
    \begin{aligned}
    \nabla_{\theta_1} MI^{\bm{\pi}}(S_2';S_1,A_1|S_2,A_2) =& \sum_{\vs,\va,s_2' \in(S, A, S_2)} \nabla_{\theta_1} (p^{\bm{\pi}}(\vs,\va,s_2')) \log \frac{p(s_2'|\vs,\va)} {p^{\bm{\pi}}(s_2'|s_2,a_2)} \\
    +& \sum_{\vs,\va,s_2' \in(S, A, S_2)} p^{\bm{\pi}}(\vs,\va,s_2') \nabla_{\theta_1} \log\frac{p(s_2'|\vs,\va)} {p^{\bm{\pi}}(s_2'|s_2,a_2)}.
    \end{aligned}\label{equ:first_step_of_gradient_of_mi}
\end{equation}

While the second term is an expectation over the trajectory and can be shown to be zero (see Appendix~\ref{appendix:mi}), it is unwieldy to deal with the first term because it requires the gradient of the stationary distribution, which depends on the policies and the dynamics of the environment. Fortunately, the gradient can still be estimated purely from sampled trajectories by drawing inspiration from the proof of the policy gradient theorem~\citep{sutton2000policy}.

The resulting policy gradient update is:
\begin{equation}\label{equ:eiti_update}
    \begin{aligned}
    \nabla_{\theta_1} J_{\theta_1}(t) =  \left(\hat{R}_1^t-\hat{V}_1^{\bm{\pi}}(s_t)\right)\nabla_{\theta_1}\log\pi_{\theta_1}(a_1^t|s_1^t)
    \end{aligned}
\end{equation}
where $\hat{V}_1^{\bm{\pi}}(s_t)$ is an augmented value function of $\hat{R}_1^t=\sum_{t'=t}^h\hat{r}_1^{t'}$ and
\begin{equation}\label{equ:eiti_reward}
\hat{r}_1^t = r^t + u_1^t + \beta \log \frac{p(s_2^{t+1}|s_1^t,s_2^t,a_1^t,a_2^t)}{p(s_2^{t+1} | s_2^t, a_2^t)}.
\end{equation}
The third term, which we call \emph{EITI reward} henceforth, is $0$ when the agents are transition-independent, \ie, when $p(s_2^{t+1}|s_1^t,s_2^t,a_1^t,a_2^t) = p(s_2^{t+1} | s_2^t, a_2^t)$, and is positive when $s_1, a_1$ increase the probability of agent 2 translating to $s_2'$. Therefore, the EITI reward is an intrinsic motivation that encourages agent 1 to visit more frequently the state-action pairs where it can influence the trajectory of agent 2. The estimation of $p(s_2^{t+1}|s_1^t,s_2^t,a_1^t,a_2^t)$ and $p(s_2^{t+1} | s_2^t, a_2^t)$ are discussed in Appendix~\ref{appendix:p_d_p}. We assume that agents know the states and actions of other agents, but this information is only available during centralized training. When execution, agents only have access to their local observations.

\subsection{Exploration via Decision-Theoretic Influence}\label{sec:voi}
Mutual information characterizes the influence of one agent's trajectory on that of the other and captures interactions between the transition functions of the agents. However, it does not provide the value of these interactions to identify interactions related to more internal and external rewards ($\tilde{r}$). To address this issue, we propose \emph{exploration via decision-theoretic influence} (EDTI) based on a decision-theoretic measure of $I$, called \emph{Value of Interaction} (VoI), which disentangles both transition and reward influences. VoI is defined as the expected difference between the action-value function of one agent (e.g., agent 2) and its counterfactual action-value function without considering the state and action of the other agent (e.g., agent 1):
\begin{equation}
\begin{aligned}
	VoI_{2|1}^{\bm{\pi}}(S_2';S_1,A_1|S_2,A_2) =& \sum_{\vs,\va,s_2'\in(S,A,S_2)} p^{\bm{\pi}}(\vs,\va,s_2') \left[Q_2^{\bm{\pi}}(\vs,\va,s_2') - Q_{2|1}^{\bm{\pi}, *}(s_2,a_2,s_2')\right], \label{equ:original_form_of_VoI}
\end{aligned}
\end{equation}
where $Q_2^{\bm{\pi}}(\vs, \va, s_2')$ is the expected rewards (including intrinsic rewards) of agent 2 defined as:
\begin{equation}
\begin{aligned}
    Q_2^{\bm{\pi}}(\vs,\va,s_2') = \tilde{r}_2(\vs,\va) + \gamma \sum_{s_1'} p(s_1'|\vs, \va, s_2')V_2^{\bm{\pi}}(\vs'),
\end{aligned}
\end{equation}
and the counterfactual action-value function $Q_2^{\bm{\pi}, *}$ (also includes intrinsic and extrinsic rewards) can be obtained by marginalizing out the state and action of agent 1:
\begin{equation}\label{equ:counterfactual_q}
	Q_{2|1}^{\bm{\pi}, *}(s_2,a_2,s_2') = \sum_{s_1^*, a_1^*} p^{\bm{\pi}}(s_1^*, a_1^* | s_2, a_2) [\tilde{r}_2(s_1^*, s_2, a_1^*, a_2) + \gamma \sum_{s_1'} p(s_1'|s_1^*,s_2,a_1^*,a_2, s_2') V_2^{\bm{\pi}}(\vs')].
\end{equation}
Note that the definition of VoI is analogous to that of MI and the difference lies in that $\log p(\cdot)$ measures the amount of information while $Q$ measures the action value. Although VoI can be obtained by learning $Q_2^{\bm{\pi}}(\vs, \va)$ and $Q_2^{\bm{\pi}}(s_2, a_2)$ and calculating the difference, we propose to use a counterfactual approach, which explicitly marginalizes out $s_1^*$ and $a_1^*$ utilizing the estimated model transition probability $p^{\bm{\pi}}(s_2' | s_2, a_2)$ and $p(s_2' |\vs, \va)$ to get a more accurate value estimate~\citep{feinberg2018model}. The performances of these two formulations are compared in the experiments.

Value functions $Q$ and $V$ used in VoI contains both expected \emph{external} rewards and \emph{internal} rewards, which will not only encourage coordinated exploration by the influence between intrinsic rewards but also filter out meaningless interactions which will not lead to extrinsic reward after intrinsic reward diminishes. To facilitate the optimization of VoI, we rewrite it under the expectation over state-action trajectories.
% \begin{restatable}{Theorem}{thovoi}
% Value of Interaction of agent 1 on agent 2 is:
    % \begin{equation}
    % \begin{aligned}
    %     VoI_{2|1}^{\bm{\pi}}(S_2';S_1,A_1|S_2,A_2) = \mathbb{E}_\tau\left[\tilde{r}_2(\bm{s}, \bm{a}) - \tilde{r}_2^{\bm{\pi}}(s_2, a_2)
    %     + \gamma \left(1 - \frac{p^{\bm{\pi}}(s_2' | s_2, a_2)}{p(s_2' |\bm{s}, \bm{a})}\right) V_2^{\bm{\pi}}(\bm{s}')\right],
    % \end{aligned}
    % \end{equation}
    %     where $\tilde{r}_2^{\bm{\pi}}(s_2, a_2)$ is the counterfactual immediate reward.
    % \label{tho:voi}\end{restatable}
\begin{equation}
    \begin{aligned}\label{equ:voi_definition}
        VoI_{2|1}^{\bm{\pi}}(S_2';S_1,A_1|S_2,A_2) = \mathbb{E}_\tau\left[\tilde{r}_2(\vs, \va) - \tilde{r}_2^{\bm{\pi}}(s_2, a_2)
        + \gamma \left(1 - \frac{p^{\bm{\pi}}(s_2' | s_2, a_2)}{p(s_2' |\vs, \va)}\right) V_2^{\bm{\pi}}(\vs')\right],
    \end{aligned}
\end{equation}
where $\tilde{r}_2^{\bm{\pi}}(s_2, a_2)$ is the counterfactual immediate reward. The detailed proof is deferred to Appendix \ref{appendix:voi_definition_proof}. From this definition, we can intuitively see how VoI reflects the value of interactions. $\tilde{r}_2(\vs, \va) - \tilde{r}_2^{\bm{\pi}}(s_2, a_2)$ and $1 - p^{\bm{\pi}}(s_2' | s_2, a_2)/p(s_2' |\vs, \va)$ measure the influence of agent 1 on the immediate reward and the transition function of agent 2, and $V_2^{\bm{\pi}}(\vs')$ serves as a scale factor in terms of future value. Only when agent 1 and agent 2 are both transition- and reward-independent, \ie, when $p^{\bm{\pi}}(s_2' | s_2, a_2)=p(s_2' | \vs, \va)$ and $r_2^{\bm{\pi}}(s_2, a_2) = r_2(\vs, \va)$ will VoI equal to 0. In particular, maximizing VoI with respect to policy parameters $\theta_1$ will lead agent 1 to meaningful interaction points, where $V_2^{\bm{\pi}}(\vs')$ is high and $s_1$, $a_1$ can increase the probability that $\vs'$ is reached.  %% Thus, we consider VoI as an intrinsic social reward and optimizing the policy with it improves team performance.

In this learning framework, agents initially explore the environment individually driven by its own curiosity, during which process they will discover potentially valuable interaction points where they can influence the transition function and (intrinsic) rewarding structure of each other. VoI highlights these points and encourages agents to visit these configurations more frequently. As intrinsic reward diminishes, VoI can gradually distinguish those interaction points which are necessary to get extrinsic rewards.
%interactions exist in intrinsic rewards
\subsubsection{Policy Optimization with VoI}\label{sec:optimize_VoI}
% The V term in VoI settles in and filter out useless interaction point and strength the role of positive ones.
We want to optimize $J_{\theta_i}$ with respect to the policy parameters $\theta_i$, where the most cumbrous term is $\nabla_{\theta_i}VoI_{-i | i}$. For brevity, we can consider a two-agent case, e.g., optimizing $VoI_{2|1}$ with respect to the policy parameters $\theta_1$. Directly computing the gradient of $\nabla_{\theta_1} VoI_{2|1}$ is not stable, because $VoI_{2|1}$ contains policy-dependent functions $\tilde{r}_2^{\bm{\pi}}(s_2, a_2), p^{\bm{\pi}}(s_2' | s_2, a_2)$, and $V_2^{\bm{\pi}}(\vs')$ (see Eq.~\ref{equ:voi_definition}). To stabilize training , we use target functions to approximate these policy-dependent functions, which is a common technique used in deep RL \citep{mnih2015human}. With this approximation, we denote
\begin{equation}
        g_2(\vs,\va) = \tilde{r}_2(\vs,\va) - 
        \tilde{r}^-_2(s_2, a_2) + \gamma  \sum_{\vs'} T(\vs'|\vs, \va) \left(1 - \frac{p^-(s_2' | s_2, a_2)}{p(s_2'|\vs, \va)}\right)V^-_2(s_1', s_2').
\end{equation}
where $r^-_2$, $p^-$, and $V^-_2$ are corresponding target functions. As these target functions are only periodically updated during the learning, their gradients over $\theta_1$ can be approximately ignored. Therefore, from Eq.~\ref{equ:voi_definition}, we have 
\begin{equation}
    \nabla_{\theta_1}VoI_{2|1}^{\bm{\pi}}(S_2';S_1,A_1|S_2,A_2) \approx \sum_{\vs,\va\in (S,A)}\left(\nabla_{\theta_1} p^{\bm{\pi}}(\vs,\va)\right)g_2(\vs,\va).
\end{equation}
% \begin{equation}
%     \begin{aligned}
%          \nabla_{\theta_1}VoI_{2|1}^{\bm{\pi}}(S_2';S_1,A_1|S_2,A_2) & \approx \sum_{\bm{s},\bm{a}\in (S,A)}p^{\bm{\pi}}(\bm{s},\bm{a})\left(\nabla_{\theta_1}\log p^{\bm{\pi}}(\bm{s},\bm{a})\right)g_2(\bm{s},\bm{a}) \\
%          &= \mathbb{E}_{\pi_1}\left[Q_{g_2}(\bm{s},\bm{a})\nabla_{\theta_1}\log\pi_{\theta_1}(a_1|s_1)\right],
%     \end{aligned}
% \end{equation}
% where $Q_{g_2}(\bm{s},\bm{a})$ is an action-value function of $r_{g_2}(\bm{s},\bm{a})=g_2(\bm{s},\bm{a})$.
Similar to the calculation of $\nabla_{\theta_i} MI$, we get the gradient at every step (see Appendix~\ref{appendix:gradient_of_VoI} for proof):
\begin{equation}\label{equ:edti_gradient}
    \nabla_{\theta_1} J_{\theta_1}(t) \approx \left(\hat{R}_1^t-\hat{V}_1^{\bm{\pi}}(s_t)\right)\nabla_{\theta_1}\log\pi_{\theta_1}(a_1^t|s_1^t),
\end{equation}
where $\hat{V}_1^{\bm{\pi}}(s_t)$ is an augmented value function regressed towards $\hat{R}_1^t=\sum_{t'=t}^h\hat{r}_1^{t'}$ and
\begin{equation}\label{equ:edti_reward}
\hat{r}_1^t = r^t + u_1^t + \beta \left[u_2^t + \gamma \left(1 - \frac{p^-(s_2^{t+1} | s_2^t, a_2^t)}{p(s_2^{t+1}|s_1^t,s_2^t,a_1^t,a_2^t)}\right) V^-_2(s_1^{t+1}, s_2^{t+1}) \right].
\end{equation}
We call $u_2^t + \gamma \left(1 - \frac{p^-(s_2^{t+1} | s_2^t, a_2^t)}{p(s_2^{t+1}|s_1^t,s_2^t,a_1^t,a_2^t)}\right) V^-_2(s_1^{t+1}, s_2^{t+1})$ the \emph{EDTI reward}.
\subsection{Discussions}\label{sec:method_discussion}
\textbf{Scale to Large Settings:} For cases with more than two agents, the VoI of agent $i$ on other agents can be defined similarly to Eq.~\ref{equ:original_form_of_VoI}, which can be annotated with $VoI_{-i | i}^{\bm{\pi}}(S_{-i}'; S_i, A_i | S_{-i}, A_{-i})$,  where $S_{-i}$ and $A_{-i}$ are the state and action sets of all agents other than $i$. In practice, agents interaction can often be decomposed to pairwise interaction so $VoI_{-i | i}^{\bm{\pi}}(S_{-i}'; S_i, A_i | S_{-i}, A_{-i})$ is well approximated by the sum of values of pairwise value of interaction:
\begin{equation}
    VoI_{-i | i}^{\bm{\pi}}(S_{-i}'; S_i, A_i | S_{-i}, A_{-i}) \approx \sum_{j\in } VoI_{j | i}^{\bm{\pi}}(S_{j}'; S_i, A_i | S_{-i}, A_{-i}).
\end{equation}

\textbf{Relationship between EITI and EDTI:} EITI and EDTI gradient updates are respectively obtained by information- and decision-theoretical influence and so, it is nontrivial to find that the first part of the EDTI reward is, in fact, a lower bound of the EITI reward:
\begin{equation}
    1- \frac{p(s_{-i}'|s_{-i},a_{-i})}{p(s_{-i}'|\vs,\va)} \leq \log \frac{p(s_{-i}'|\vs,\va)}{p(s_{-i}'|s_{-i},a_{-i})}, \ \ \ \forall \vs, \va, s_{-i}'
\end{equation}
which easily follows given that $\log x \geq 1-1/x$ for $\forall x>0$. This draws a connection between term $1- p(s_{-i}'|s_{-i},a_{-i})/p(s_{-i}'|\vs,\va)$ and optimizing a part of EDIT objective function is optimizing a lower bound of that of EITI in the meantime. 
%This draws a connection between term $1- p(s_{-i}'|s_{-i},a_{-i})/p(s_{-i}'|\vs,\va)$ and ... transition interactions between agents, as discussed in Sec.~\ref{sec:voi}.

\textbf{Compare EDTI to Centralized Methods:} Different from a centralized method which directly includes value functions of other agents in the optimization objective, (by setting total reward $\hat{r}_i = r + u_i + \beta (u_{-i} + \gamma V_{-i})$, call this method \emph{plusV} henceforth), EDTI is, in fact, a kind of \emph{intrinsic reward distribution}. Intuitively, both individual exploration processes of other agents and their interaction with agent $i$ contribute to $V_{-i}$. PlusV does not distinguish between these two effects while EDTI measures the change in expected future (intrinsic) rewards given the state and action of agent $i$.

The difference comes from the counterfactual form of VoI, which is a kind of \emph{difference rewards} that has been frequently used to assign external rewards in the multi-agent deep RL literature~\citep{wolpert2002optimal, foerster2018counterfactual, nguyen2018credit}. This counterfactual formulation helps individual agents filter out the noise in global and other agents' intrinsic reward signals (which include effects from other agents' explorations), and assesses their individual contribution to the gaining of other agents' intrinsic and environmental rewards. In our experiments, we empirically compare the performance of plusV against our methods.
% Substantial literature has proven the effectiveness of such difference rewards~\citep{wolpert2002optimal, foerster2018counterfactual, nguyen2018credit} in reward assignment.
% centralized training decentralized execution ... 

\begin{table} [t]
    \caption{Baseline algorithms. The third column is the reward used to train the value function of PPO. $u_i$ and $u_{cen}$ are curiosity about individual state $s_i$ and global state $\vs$, $T_1=\log\left(p(s'_{\shortn i}|\vs,\va) / p(s'_{\shortn i}|s_{\shortn i},a_{\shortn i})\right)$, $T_2=1-p(s'_{\shortn i}|s_{\shortn i},a_{\shortn i}) / p(s'_{\shortn i}|\vs,\va)$, and $\Delta Q_{\shortn i}(\vs, \va) = Q_{\shortn i}(\vs, \va)-Q_{\shortn i}(s_{\shortn i}, a_{\shortn i})$. Social influence~\citep{jaques2018intrinsic} and COMA~\citep{foerster2018counterfactual} are augmented with curiosity.}
    \label{tab:baselines}
    \centering
    \begin{tabular}{CRCRCRCR}
        \toprule
        \multicolumn{2}{c}{} &
        \multicolumn{2}{l}{Alg.} &
        \multicolumn{2}{l}{Reward} &
        \multicolumn{2}{l}{Description} \\
        \cmidrule(lr){1-2}
        \cmidrule(lr){3-4}
        \cmidrule(lr){5-6}
        \cmidrule(lr){7-8}
        
        \multicolumn{2}{c}{\multirow{2}{*}{Ours}} & \multicolumn{2}{l}{EITI}  & \multicolumn{2}{l}{$r + u_i + \beta T_1$} & \multicolumn{2}{l}{Influence-theoretic influence} \\
        \multicolumn{2}{c}{} & \multicolumn{2}{l}{EDTI}  & \multicolumn{2}{l}{$r + u_i + \beta(u_{\scalebox{0.7}[1.0]{-}i} + \gamma T_2 V_{\scalebox{0.7}[1.0]{-}i})$} & \multicolumn{2}{l}{Decision-theoretic influence} \\
        \midrule
        \multicolumn{2}{c}{\multirow{4}{*}{\makecell{Other\\ Exploration \\ Methods}}} & \multicolumn{2}{l}{random} &  \multicolumn{2}{l}{$r$} & \multicolumn{2}{l}{Pure PPO} \\
        \multicolumn{2}{c}{} & \multicolumn{2}{l}{cen} & \multicolumn{2}{l}{$r + u_{cen}$} & \multicolumn{2}{l}{Decentralized PPO with cen curiosity} \\
        \multicolumn{2}{c}{} & \multicolumn{2}{l}{dec} & \multicolumn{2}{l}{$r + u_i$} & \multicolumn{2}{l}{Decentralized PPO with dec curiosity} \\
        \multicolumn{2}{c}{} & \multicolumn{2}{l}{cen\_control} & \multicolumn{2}{l}{$r + u_{cen}$} & \multicolumn{2}{l}{Centralized PPO with cen curiosity} \\
        \cmidrule(lr){1-2}
        \cmidrule(lr){3-4}
        \cmidrule(lr){5-6}
        \cmidrule(lr){7-8}
        \multicolumn{2}{c}{\multirow{4}{*}{Ablations}} & \multicolumn{2}{l}{r\_influence} & \multicolumn{2}{l}{$r + u_i + \beta u_{\scalebox{0.7}[1.0]{-}i}$} & \multicolumn{2}{l}{Disentangle reward interaction} \\
        \multicolumn{2}{c}{} & \multicolumn{2}{l}{plusV} & \multicolumn{2}{l}{$r + u_i + \beta V_{\scalebox{0.7}[1.0]{-}i}$} & \multicolumn{2}{l}{Use other agents' value functions} \\
        \multicolumn{2}{c}{} & \multicolumn{2}{l}{shared\_critic} & \multicolumn{2}{l}{$r + u_{cen}$} & \multicolumn{2}{l}{PPO with shared $V$ and cen curiosity} \\
        \multicolumn{2}{c}{} & \multicolumn{2}{l}{Q-Q} & \multicolumn{2}{l}{$r + u_i + \beta \Delta Q_{\shortn i}(\vs, \va)$} & \multicolumn{2}{l}{EDTI without explicit counterfactual} \\
        \cmidrule(lr){1-2}
        \cmidrule(lr){3-4}
        \cmidrule(lr){5-6}
        \cmidrule(lr){7-8}
        \multicolumn{2}{c}{\multirow{3}{*}{\makecell{Related\\ Works}}} &  \multicolumn{2}{l}{social} & \multicolumn{2}{l}{---} & \multicolumn{2}{l}{By \citet{jaques2018intrinsic}} \\
        \multicolumn{2}{c}{} & \multicolumn{2}{l}{COMA} & \multicolumn{2}{l}{---} & \multicolumn{2}{l}{By \citet{foerster2018counterfactual}} \\
        \multicolumn{2}{c}{} & \multicolumn{2}{l}{Multi} & \multicolumn{2}{l}{---} & \multicolumn{2}{l}{By \citet{iqbal2019coordinated}} \\
        \toprule
    \end{tabular}
\end{table}
\section{Experimental Results}
Our experiments aim to answer the following questions: (1) Can EITI and EDTI rewards capture interaction points? If they can, how do these points change throughout exploration? (2) Can exploiting these interaction points facilitate exploration and learning performance? (3) Can EDTI filter out interaction points that are not related to environmental rewards? (4) What if only reward influence between agents are disentangled? We evaluate our approach on a set of multi-agent tasks with sparse rewards based on a discrete version of multi-agent particle world environment~\citep{lowe2017multi}. PPO~\citep{schulman2017proximal} is used as the underlying algorithm. For evaluation, all experiments are carried out with 5 different random seeds and results are shown with $95\%$ confidence interval. Demonstrative videos\footnote{\url{https://sites.google.com/view/influence-based-ma-exploration/}} are available online.
% \begin{table}[th]
%     \caption{Caption}
%     \label{tab:my_label}
%     \begin{tabular}{c|c|c|c|c|c|c|c}
%         \hline
%         Class    &\multicolumn{2}{c|}{\textbf{Ours}} & \multicolumn{4}{c}{\textbf{Ablations}} \\ 
%         Alg. & EITI & EDIT & plus\_r & plus\_v & Shared critic & V-V  \\  
%         Reward & Eq.~\ref{equ:eiti_reward} & Eq.~\ref{equ:edti_reward} & $\tilde{r}_i + u_{\scalebox{0.7}[1.0]{-}i}$ & $\tilde{r}_i + \beta(r_{\scalebox{0.7}[1.0]{-}i} + \gamma V_{\scalebox{0.7}[1.0]{-}i})$ & $r + u_{cen}$ & $\tilde{r}_i$  \\
%         Desc. &     &      &         &         &   \\
        
%         \hline
%         Class    & \multicolumn{4}{c|}{\textbf{Other Exploration Methods}} & \textbf{Social influence} & \\
%         Alg. & Random & Dec & Cen & Cen control & \citet{jaques2018intrinsic} \\
%         Reward & $r$ & $\tilde{r}_i$ & $r + u_{cen}$ & $r + u_{cen}$ & Influence MOA\\ 
%         Description & \\
%     \end{tabular}
% \end{table}
\textbf{Baselines} We compare our methods with various baselines shown in Table~\ref{tab:baselines}. In particular, we carry out the following ablation studies: i) r\_influence disentangles immediate reward influence between agents, (derivation of the associated augmented reward can be found in Appendix~\ref{appendix:r_influence}. Reward influence in long term is not considered because it inevitably involves transition interactions) ii) PlusV as described in Sec.~\ref{sec:method_discussion}. iii) Shared\_critic uses decentralized PPO agents with shared centralized value function and thus is a cooperative version of MADDPG~\citep{lowe2017multi} augmented with intrinsic reward of curiosity. iv) Q-Q is similar to EDTI but without explicit counterfactual formulation, as described in Sec.~\ref{sec:voi}. We also note that EITI is an ablation of EDTI which considers transition interactions. PlusV, shared\_critic, Q-Q, and cen\_control have access to global or other agents' value functions during training. When execution, all the methods except cen\_control only require local state.
% \begin{figure}
%     \centering
%     \includegraphics[height=0.26\linewidth]{fig:env/pass.pdf}\hfill
%     \includegraphics[height=0.26\linewidth]{fig:env/pass_optimal_policy.pdf}\hfill
%     \includegraphics[height=0.26\linewidth]{fig:env/secret_room.pdf}
%     \caption{Didactic examples. Left: task \textbf{pass}. Two agents starting at the upper-left corner are only rewarded when both of them reach the other room through the door, which will open only when at least one of the switches is occupied by one or more agents. Middle: the optimal cooperative strategy to solve Pass; Right: \textbf{secret-room}. An extension of pass with 4 rooms and switches. When the switch 1 is occupied, all the three doors turn open. And the three switches on the right only control the door of its room. The agents need to reach the upper right room to achieve any reward. }
%     \label{fig:didactic_examples}
% \end{figure}
\begin{figure}
    \centering
    \includegraphics[height=0.27\linewidth]{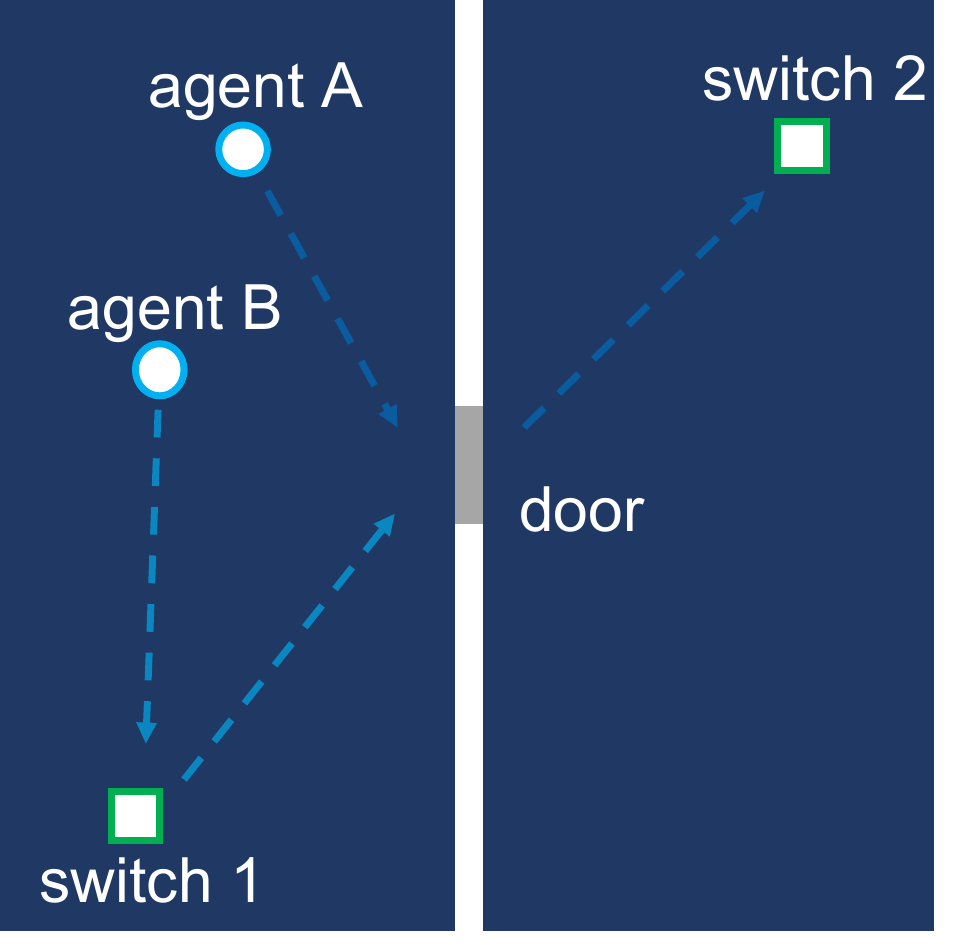}\hfill
    \includegraphics[height=0.27\linewidth]{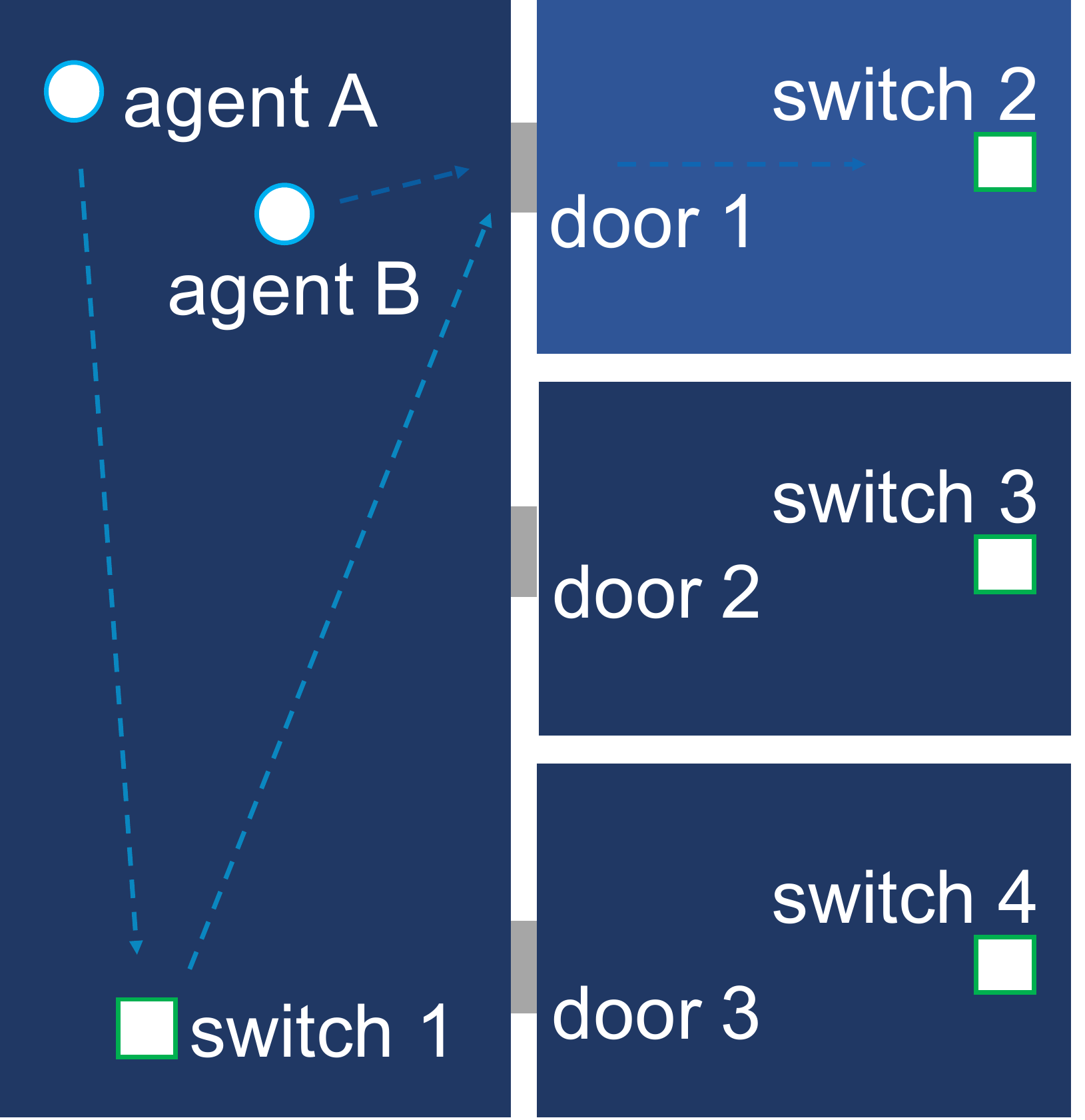}\hfill
    \includegraphics[height=0.27\linewidth]{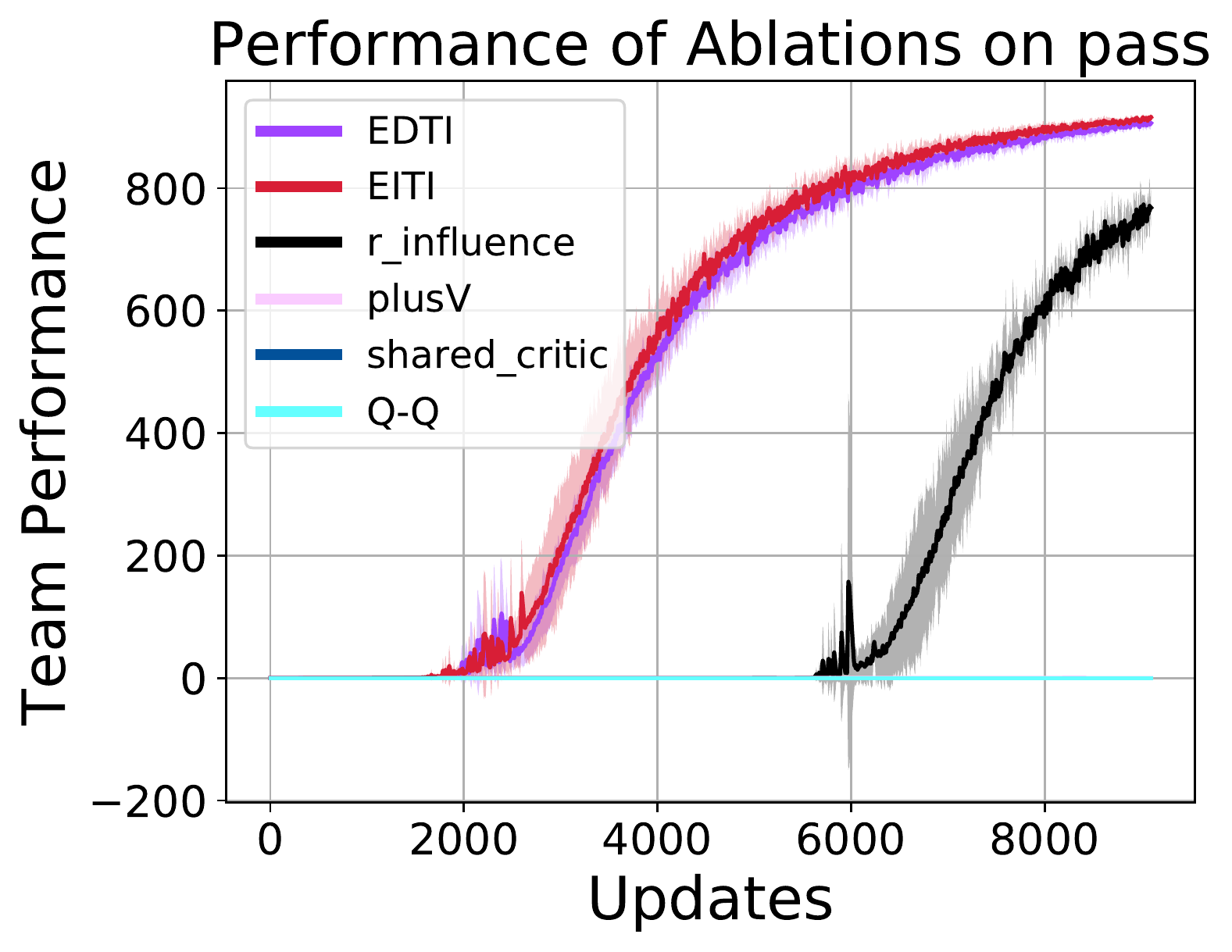}
    \caption{Didactic examples. Left: task \textbf{pass}. Two agents starting at the upper-left corner are only rewarded when both of them reach the other room through the door, which will open only when at least one of the switches is occupied by one or more agents. Middle: \textbf{secret-room}. An extension of pass with 4 rooms and switches. When the switch 1 is occupied, all the three doors turn open. And the three switches on the right only control the door of its room. The agents need to reach the upper right room to achieve any reward. Right: comparison of our methods with ablations on pass.}
    \label{fig:didactic_examples}
\end{figure}
\subsection{Didactic Examples}
We present two didactic examples of multi-agent cooperation tasks with sparse reward to explain how EITI and EDTI work. The first didactic example consists of a $30\times 30$ maze with two rooms and a door with two switches (Fig.~\ref{fig:didactic_examples} left). In the optimal strategy, one agent should first step on switch 1 to help the other agent pass the door, and then the agent that has already reached the right half should further go to switch 2 to bring the remaining agent in. There are two pairs of interaction points in this task: (switch 1, door) and (switch 2, door), \ie, transition probability of the agent near door is determined by whether another agent is on one of the switch.

Fig.~\ref{fig:didactic_examples}-right and Fig.~\ref{fig:pass}-top show the learning curves of our methods and all the baselines, among which EITI, EDTI, r\_influence, Multi, and centralized control can learn the winning strategy and ours learn much more efficiently. Fig.~\ref{fig:pass}-bottom gives a possible explanation why our methods work. EITI and EDTI rewards successfully highlight the interaction points (before 100 and 2100 updates, respectively). Agents are encouraged to explore these configurations more frequently and thus have better chance to learn the goal strategy. EDTI reward considers the value function of the other agent, so it converges slower than the EITI reward. In contrast, directly adding the other agent's intrinsic rewards and value functions is noisy (see "plusV reward") and confuses the agent because these contain the effect of the other agent's exploration. As for centralized control, global curiosity encourages agents to try all possible configurations, so it can find environmental rewards in most tasks. However, visiting all configurations without bias renders it inefficient -- external rewards begin to dominate the behaviors of agents after 7000 updates even with the help of centralized learning algorithm. Our methods use the same information as centralized exploration but take advantages of agents' interactions to accelerate exploration.

In order to evaluate whether EDTI can help filter out noisy interaction points and accelerate exploration, we conduct experiments in a second didactic task (see Fig.~\ref{fig:didactic_examples} middle). It is also a navigation task in a $25\times 25$ maze where agents are rewarded for being in a goal room. However, in this experiment, we consider a case where there are four rooms and the upper right one is attached to reward. This task contains $6$ pairs of interaction points (switch 1 with each of the doors, each switch with the door of the same room), but only two of them are related to external rewards, \ie, (switch 1, door 1) and (switch 2, door 1). As Fig.~\ref{fig:threepass}-right shows, EITI agents treat three doors equally even after 7400 updates (see Fig.~\ref{fig:threepass} right, 7400 updates, top row). In comparison, although EDTI reward suffers from noise in the beginning, it clearly highlight two pairs of valuable interaction points (see Fig.~\ref{fig:threepass} right, 7400 updates, bottom row) as intrinsic reward diminishes. This can explain why EDTI outperforms EITI (Fig.~\ref{fig:threepass} left).
\begin{figure}
    \centering
    \includegraphics[width=\linewidth]{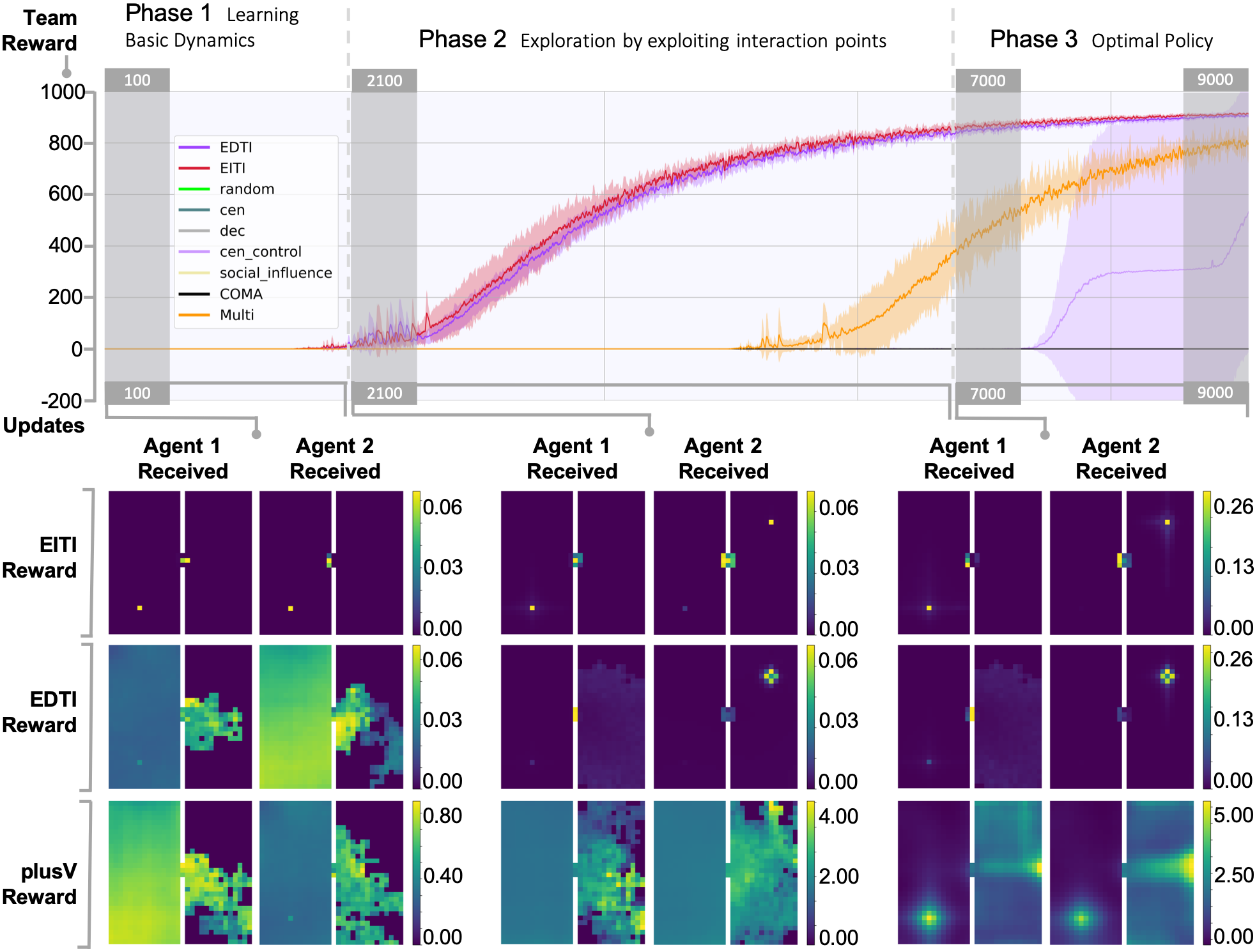}
    \caption{Development of performance of our methods compared to baselines and intrinsic reward terms of EITI, EDTI, and plusV over the training period of 9000 PPO updates segmented into three phases. "Team Reward" shows averaged team reward gained in a episode, with a maximum of 1000. It shows that only EITI, EDTI, and centralized control and Multi can learn the strategy during this stage. "EITI reward", "EDTI reward", and "plusV reward" demonstrate the evolving of corresponding intrinsic rewards.}
    \label{fig:pass}
\end{figure}
\begin{figure}
    \centering
    \includegraphics[height=0.22\linewidth]{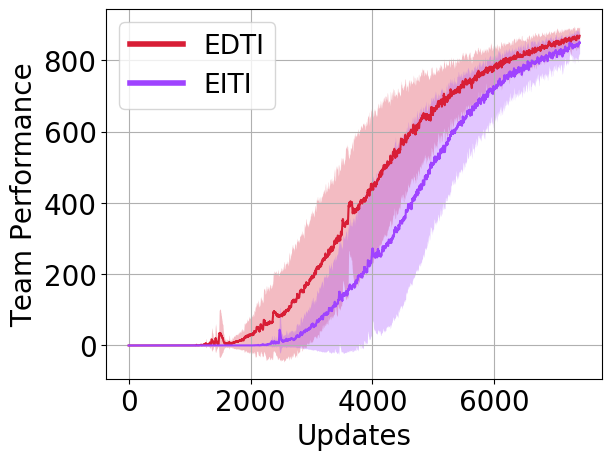}\hfill
    \includegraphics[height=0.22\linewidth]{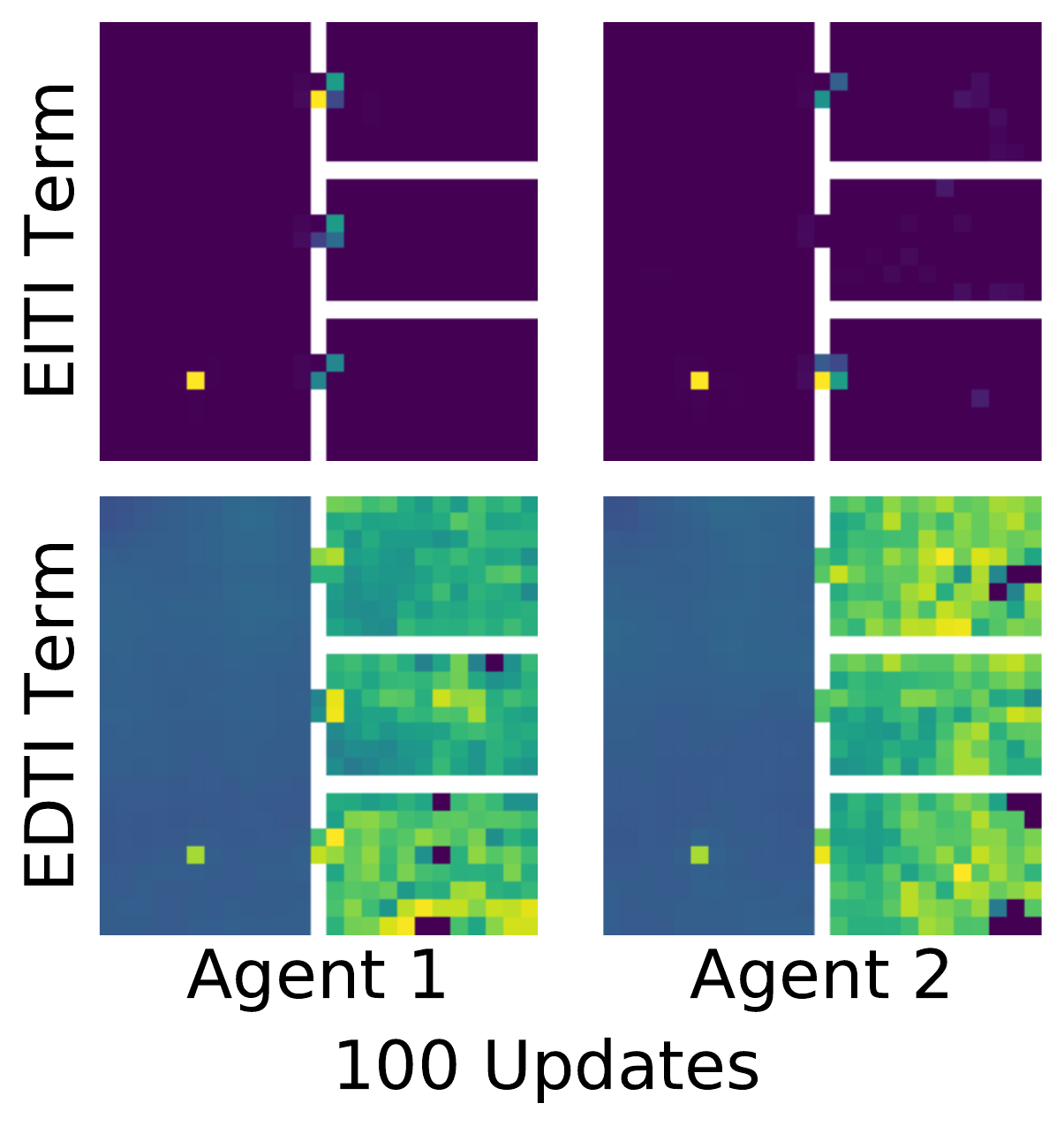}\hfill
    \includegraphics[height=0.22\linewidth]{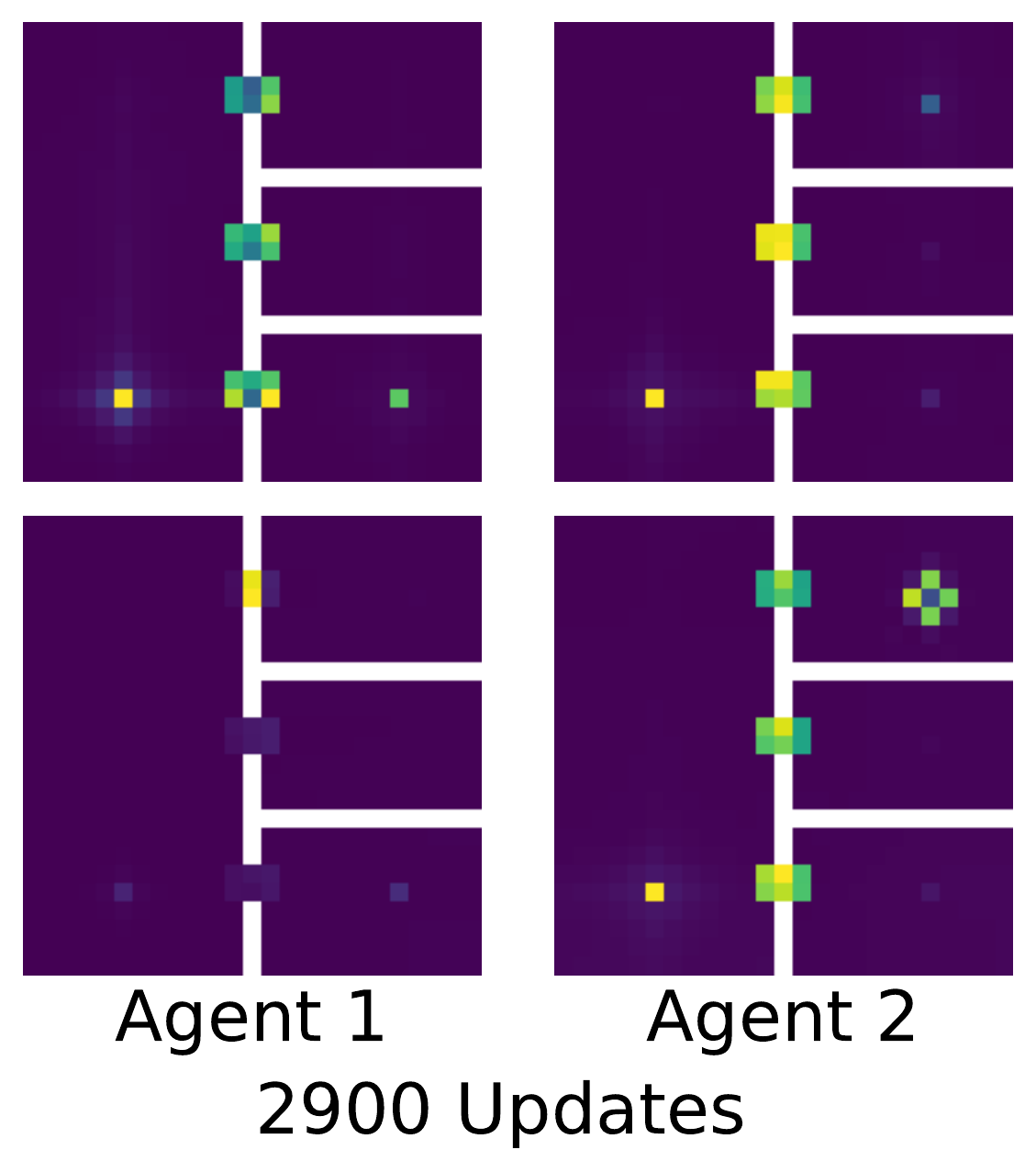}\hfill
    \includegraphics[height=0.22\linewidth]{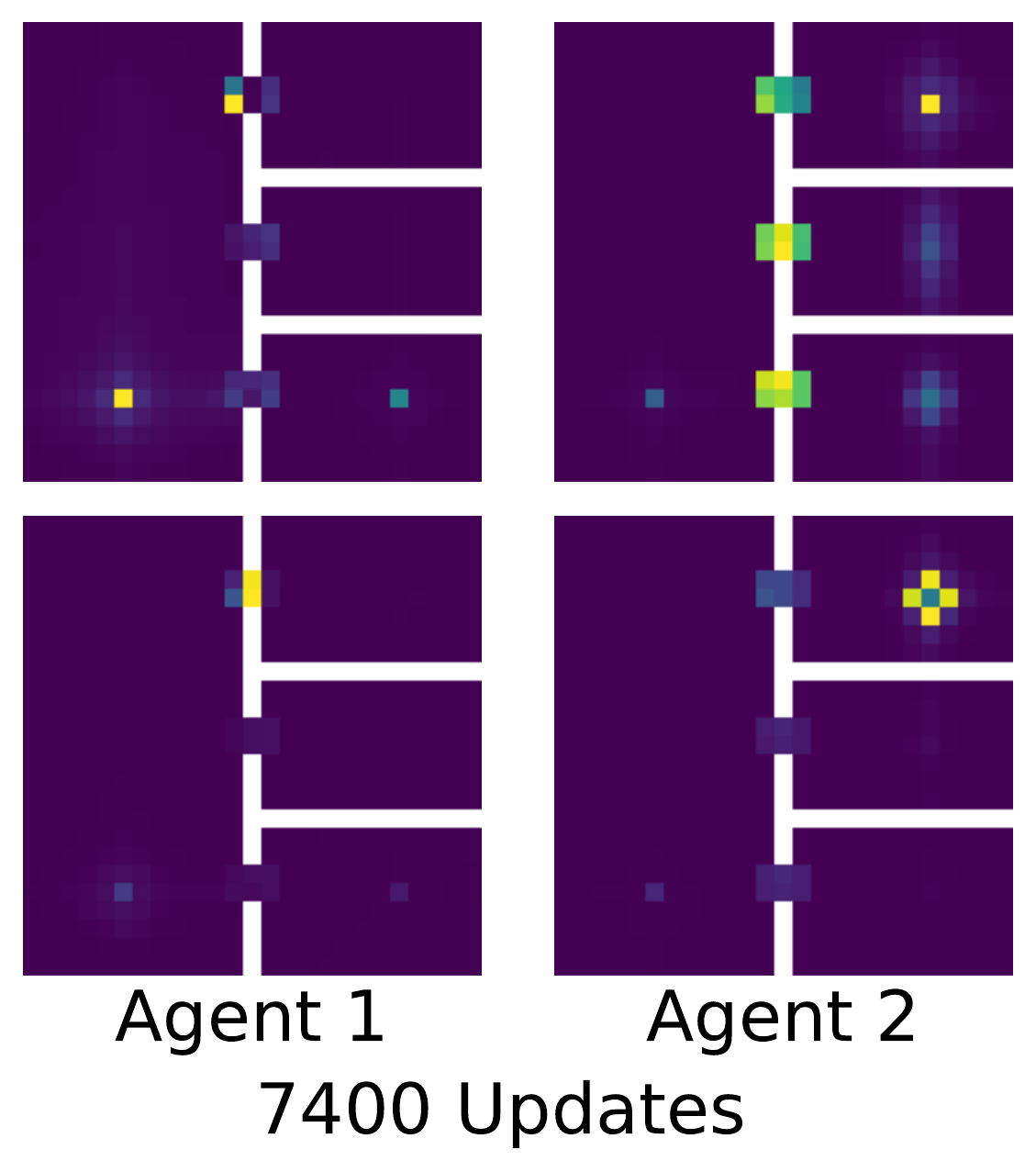}
    \includegraphics[height=0.22\linewidth]{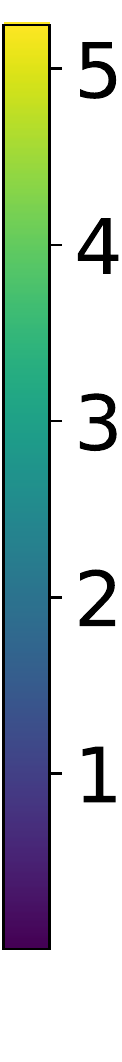}
    \caption{Left: performance comparison between EDTI and EITI on secret-room over 7400 PPO updates. Right: EITI and EDTI terms of two agents after 100, 2900, and 7400 updates.}
    \label{fig:threepass}
\end{figure}
\begin{figure}
    \centering
    \includegraphics[height=0.22\linewidth]{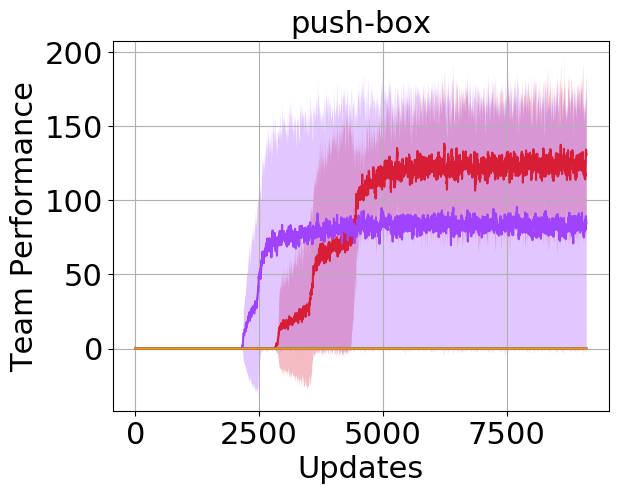}\hfill
    \includegraphics[height=0.22\linewidth]{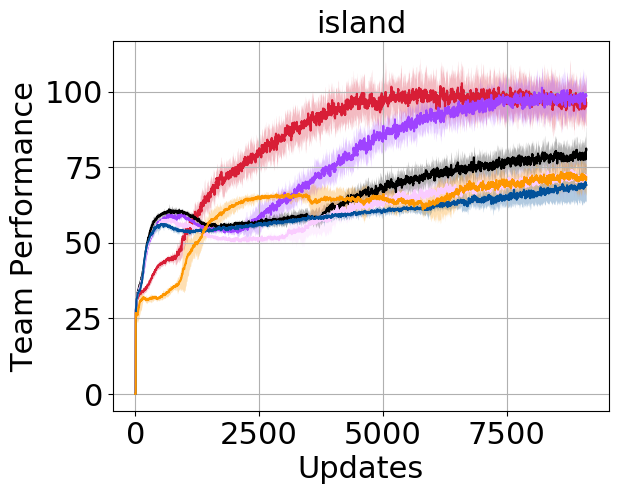}\hfill
    \includegraphics[height=0.22\linewidth]{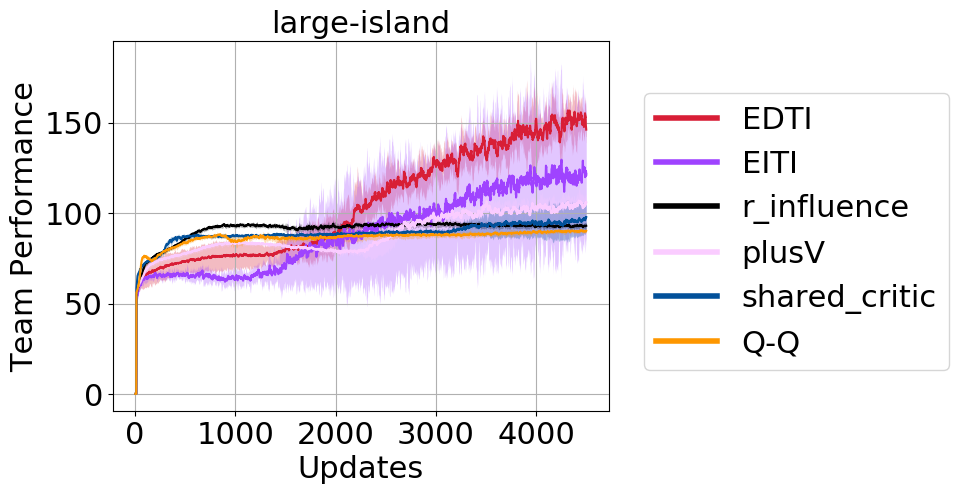}
    \caption{Comparison of our methods against ablations for push-box, island, and large-island. Comparison with baselines is shown in Fig.~\ref{fig:push_island_large_island} in Appendix~\ref{appendix:implementation_details}.}
    \label{fig:push_island_large_island_ablations}
\end{figure}
\subsection{Exploration in Complex Tasks}
Next, we evaluate the performance of our methods on more complex tasks. To this end, we use three sparse reward cooperative multi-agent tasks depicted in Fig.~\ref{fig:more_env} of Appendix~\ref{appendix:implementation_details} and analyzed below. Details of implementation and experiment settings are also described in Appendix~\ref{appendix:implementation_details}.

\textbf{Push-box:} A $15\times 15$ room is populated with 2 agents and 1 box. Agents need to push the box to the wall in 300 environment steps to get a reward of 1000. However, the box is so heavy that only when two agents push it in the same direction at the same time can it be moved a grid. Agents need to coordinate their positions and actions for multiple steps to earn a reward. The purpose of this task is to demonstrate that EITI and EDTI can explore long-term cooperative strategy.

\textbf{Island:} This task is a modified version of the classic Stag Hunt game~\citep{peysakhovich2018prosocial} where two agents roam a $10\times 10$ island populated with 9 treasures and a random walking beast for 300 environment steps. Agents can collect a treasure by stepping on it to get a team reward of 10 or, by attacking the beast within their attack range, capture it for a reward of 300. The beast would also attack the agents when they are too close. The beast and agent have a maximum energy of 8 and 5 respectively, which will be subtracted by 1 every time attacked. Therefore, an agent is too weak to beat the beast alone and they have to cooperate. In order to learn optimal strategy in this task, one method has to keep exploring after sub-optimal external rewards are found.

\textbf{Large-island:} Similar to island but with more agents (4), more treasures (16), and a beast with more energy (16) and a higher reward (600) for being caught. This task aims to demonstrate feasibility of our methods in cases with more than 2 agents.

Push-box requires agents to take coordinated actions at certain positions for multiple steps to get rewarded. Therefore, this task is particularly challenging and all the baselines struggle to earn any reward (Fig.~\ref{fig:push_island_large_island_ablations} left and Fig.~\ref{fig:push_island_large_island} left). Our methods are considerably more successful because interaction happens when the box is moved -- agents remain unmoved when they push the box alone but will move by a grid if push it together. In this way, EITI and EDTI agents are rewarded intrinsically to move the box and thus are able to quickly find the optimal policy.

In the island task, collecting treasures is a easily-attainable local optimal. However, efficient treasures collecting requires the agents to spread on the island. This leads to a situation where attempting to attack the beast seems a bad choice since it is highly possible that agents will be exposed to the beast's attack alone. They have to give up profitable spreading strategy and take the risk of being killed to discover that if they attack the beast collectively for several timesteps, they will get much more rewards. Our methods help solve this challenge by giving agents intrinsic incentives to appear together in the attack range of the beast, where they have indirect interactions (health is part of the state and it decreases slower when the two are attacked alternatively). Fig.~\ref{fig:island_details} in Appendix~\ref{appendix:implementation_details} demonstrates that our methods learn to catch the beast quickly, and thus have better performance (Fig.~\ref{fig:push_island_large_island} right).

Finally, outperformance of our methods on large-island proves that they can successfully handle cases with more than two agents.

In summary, both of our methods are able to facilitate effective exploration on all the tasks by exploiting interactions. EITI outperforms EDTI in scenarios where all interaction points align with extrinsic rewards. On other tasks, EDTI performs better than EITI due to its ability to filter out interaction points that can not lead to more values.

We also study EDTI with only intrinsic rewards, discussion and results are included in Appendix~\ref{appendix:voi_int}.

\section{Closing Remarks}
In this paper, we study the multi-agent exploration problem and propose two influence-based methods that exploits the interaction structure. These methods are based on two interaction measures, MI and \emph{Value of Interaction} (VoI), which respectively measure the amount and value of one agent's influence on the other agents' exploration processes. These two measures can be best regraded as exploration bonus distribution. We also propose an optimization method in the policy gradient framework, which enables agents to achieve coordinated exploration in a decentralized manner and optimize team performance.

\bibliography{iclr2020_conference}

\begin{thebibliography}{59}
\providecommand{\natexlab}[1]{#1}
\providecommand{\url}[1]{\texttt{#1}}
\expandafter\ifx\csname urlstyle\endcsname\relax
  \providecommand{\doi}[1]{doi: #1}\else
  \providecommand{\doi}{doi: \begingroup \urlstyle{rm}\Url}\fi

\bibitem[Agrawal \& Jia(2017)Agrawal and Jia]{agrawal2017optimistic}
Shipra Agrawal and Randy Jia.
\newblock Optimistic posterior sampling for reinforcement learning: worst-case
  regret bounds.
\newblock In \emph{Advances in Neural Information Processing Systems}, pp.\
  1184--1194, 2017.

\bibitem[Azar et~al.(2017)Azar, Osband, and Munos]{azar2017minimax}
Mohammad~Gheshlaghi Azar, Ian Osband, and R{\'e}mi Munos.
\newblock Minimax regret bounds for reinforcement learning.
\newblock In \emph{Proceedings of the 34th International Conference on Machine
  Learning-Volume 70}, pp.\  263--272. JMLR. org, 2017.

\bibitem[Bargiacchi et~al.(2018)Bargiacchi, Verstraeten, Roijers, Now{\'e}, and
  Hasselt]{bargiacchi2018learning}
Eugenio Bargiacchi, Timothy Verstraeten, Diederik Roijers, Ann Now{\'e}, and
  Hado Hasselt.
\newblock Learning to coordinate with coordination graphs in repeated
  single-stage multi-agent decision problems.
\newblock In \emph{International Conference on Machine Learning}, pp.\
  491--499, 2018.

\bibitem[Barron et~al.(2018)Barron, Obst, and Amor]{barron2018information}
Trevor Barron, Oliver Obst, and Heni~Ben Amor.
\newblock Information maximizing exploration with a latent dynamics model.
\newblock \emph{arXiv preprint arXiv:1804.01238}, 2018.

\bibitem[Barto(2013)]{barto2013intrinsic}
Andrew~G Barto.
\newblock Intrinsic motivation and reinforcement learning.
\newblock In \emph{Intrinsically motivated learning in natural and artificial
  systems}, pp.\  17--47. Springer, 2013.

\bibitem[Bellemare et~al.(2016)Bellemare, Srinivasan, Ostrovski, Schaul,
  Saxton, and Munos]{bellemare2016unifying}
Marc Bellemare, Sriram Srinivasan, Georg Ostrovski, Tom Schaul, David Saxton,
  and Remi Munos.
\newblock Unifying count-based exploration and intrinsic motivation.
\newblock In \emph{Advances in Neural Information Processing Systems}, pp.\
  1471--1479, 2016.

\bibitem[Burda et~al.(2018{\natexlab{a}})Burda, Edwards, Pathak, Storkey,
  Darrell, and Efros]{burda2018large}
Yuri Burda, Harri Edwards, Deepak Pathak, Amos Storkey, Trevor Darrell, and
  Alexei~A Efros.
\newblock Large-scale study of curiosity-driven learning.
\newblock \emph{arXiv preprint arXiv:1808.04355}, 2018{\natexlab{a}}.

\bibitem[Burda et~al.(2018{\natexlab{b}})Burda, Edwards, Storkey, and
  Klimov]{burda2018exploration}
Yuri Burda, Harrison Edwards, Amos Storkey, and Oleg Klimov.
\newblock Exploration by random network distillation.
\newblock \emph{arXiv preprint arXiv:1810.12894}, 2018{\natexlab{b}}.

\bibitem[Cao et~al.(2012)Cao, Yu, Ren, and Chen]{cao2012overview}
Yongcan Cao, Wenwu Yu, Wei Ren, and Guanrong Chen.
\newblock An overview of recent progress in the study of distributed
  multi-agent coordination.
\newblock \emph{IEEE Transactions on Industrial informatics}, 9\penalty0
  (1):\penalty0 427--438, 2012.

\bibitem[Chalkiadakis \& Boutilier(2003)Chalkiadakis and
  Boutilier]{Chalkiadakis:2003:CMR:860575.860689}
Georgios Chalkiadakis and Craig Boutilier.
\newblock Coordination in multiagent reinforcement learning: A bayesian
  approach.
\newblock In \emph{Proceedings of the Second International Joint Conference on
  Autonomous Agents and Multiagent Systems}, AAMAS '03, pp.\  709--716, New
  York, NY, USA, 2003. ACM.
\newblock ISBN 1-58113-683-8.
\newblock \doi{10.1145/860575.860689}.
\newblock URL \url{http://doi.acm.org/10.1145/860575.860689}.

\bibitem[Chentanez et~al.(2005)Chentanez, Barto, and
  Singh]{chentanez2005intrinsically}
Nuttapong Chentanez, Andrew~G Barto, and Satinder~P Singh.
\newblock Intrinsically motivated reinforcement learning.
\newblock In \emph{Advances in neural information processing systems}, pp.\
  1281--1288, 2005.

\bibitem[Dhariwal et~al.(2017)Dhariwal, Hesse, Klimov, Nichol, Plappert,
  Radford, Schulman, Sidor, Wu, and Zhokhov]{baselines}
Prafulla Dhariwal, Christopher Hesse, Oleg Klimov, Alex Nichol, Matthias
  Plappert, Alec Radford, John Schulman, Szymon Sidor, Yuhuai Wu, and Peter
  Zhokhov.
\newblock Openai baselines.
\newblock \url{https://github.com/openai/baselines}, 2017.

\bibitem[Dimakopoulou \& Van~Roy(2018)Dimakopoulou and
  Van~Roy]{dimakopoulou2018coordinated}
Maria Dimakopoulou and Benjamin Van~Roy.
\newblock Coordinated exploration in concurrent reinforcement learning.
\newblock In \emph{International Conference on Machine Learning}, pp.\
  1270--1278, 2018.

\bibitem[Dimakopoulou et~al.(2018)Dimakopoulou, Osband, and
  Van~Roy]{dimakopoulou2018scalable}
Maria Dimakopoulou, Ian Osband, and Benjamin Van~Roy.
\newblock Scalable coordinated exploration in concurrent reinforcement
  learning.
\newblock In \emph{Advances in Neural Information Processing Systems}, pp.\
  4219--4227, 2018.

\bibitem[Feinberg et~al.(2018)Feinberg, Wan, Stoica, Jordan, Gonzalez, and
  Levine]{feinberg2018model}
Vladimir Feinberg, Alvin Wan, Ion Stoica, Michael~I Jordan, Joseph~E Gonzalez,
  and Sergey Levine.
\newblock Model-based value estimation for efficient model-free reinforcement
  learning.
\newblock \emph{arXiv preprint arXiv:1803.00101}, 2018.

\bibitem[Florensa et~al.(2017)Florensa, Duan, and
  Abbeel]{florensa2017stochastic}
Carlos Florensa, Yan Duan, and Pieter Abbeel.
\newblock Stochastic neural networks for hierarchical reinforcement learning.
\newblock \emph{arXiv preprint arXiv:1704.03012}, 2017.

\bibitem[Foerster et~al.(2016)Foerster, Assael, de~Freitas, and
  Whiteson]{foerster2016learning}
Jakob Foerster, Ioannis~Alexandros Assael, Nando de~Freitas, and Shimon
  Whiteson.
\newblock Learning to communicate with deep multi-agent reinforcement learning.
\newblock In \emph{Advances in Neural Information Processing Systems}, pp.\
  2137--2145, 2016.

\bibitem[Foerster et~al.(2018)Foerster, Farquhar, Afouras, Nardelli, and
  Whiteson]{foerster2018counterfactual}
Jakob~N Foerster, Gregory Farquhar, Triantafyllos Afouras, Nantas Nardelli, and
  Shimon Whiteson.
\newblock Counterfactual multi-agent policy gradients.
\newblock In \emph{Thirty-Second AAAI Conference on Artificial Intelligence},
  2018.

\bibitem[Fox \& Roberts(2012)Fox and Roberts]{fox2012tutorial}
Charles~W Fox and Stephen~J Roberts.
\newblock A tutorial on variational bayesian inference.
\newblock \emph{Artificial intelligence review}, 38\penalty0 (2):\penalty0
  85--95, 2012.

\bibitem[Gupta et~al.(2018)Gupta, Mendonca, Liu, Abbeel, and
  Levine]{gupta2018meta}
Abhishek Gupta, Russell Mendonca, YuXuan Liu, Pieter Abbeel, and Sergey Levine.
\newblock Meta-reinforcement learning of structured exploration strategies.
\newblock In \emph{Advances in Neural Information Processing Systems}, pp.\
  5302--5311, 2018.

\bibitem[Houthooft et~al.(2016)Houthooft, Chen, Duan, Schulman, De~Turck, and
  Abbeel]{houthooft2016vime}
Rein Houthooft, Xi~Chen, Yan Duan, John Schulman, Filip De~Turck, and Pieter
  Abbeel.
\newblock Vime: Variational information maximizing exploration.
\newblock In \emph{Advances in Neural Information Processing Systems}, pp.\
  1109--1117, 2016.

\bibitem[Hughes et~al.(2018)Hughes, Leibo, Phillips, Tuyls, Due{\~n}ez-Guzman,
  Casta{\~n}eda, Dunning, Zhu, McKee, Koster, et~al.]{hughes2018inequity}
Edward Hughes, Joel~Z Leibo, Matthew Phillips, Karl Tuyls, Edgar
  Due{\~n}ez-Guzman, Antonio~Garc{\'\i}a Casta{\~n}eda, Iain Dunning, Tina Zhu,
  Kevin McKee, Raphael Koster, et~al.
\newblock Inequity aversion improves cooperation in intertemporal social
  dilemmas.
\newblock In \emph{Advances in Neural Information Processing Systems}, pp.\
  3330--3340, 2018.

\bibitem[Hyoungseok~Kim(2019)]{kim2019emi}
Yeonwoo Jeong Sergey Levine Hyun Oh~Song Hyoungseok~Kim, Jaekyeom~Kim.
\newblock Emi: Exploration with mutual information.
\newblock In \emph{Proceedings of the 36th International Conference on Machine
  Learning}. JMLR. org, 2019.

\bibitem[Iqbal \& Sha(2019)Iqbal and Sha]{iqbal2019coordinated}
Shariq Iqbal and Fei Sha.
\newblock Coordinated exploration via intrinsic rewards for multi-agent
  reinforcement learning.
\newblock \emph{arXiv preprint arXiv:1905.12127}, 2019.

\bibitem[Jaksch et~al.(2010)Jaksch, Ortner, and Auer]{jaksch2010near}
Thomas Jaksch, Ronald Ortner, and Peter Auer.
\newblock Near-optimal regret bounds for reinforcement learning.
\newblock \emph{Journal of Machine Learning Research}, 11\penalty0
  (Apr):\penalty0 1563--1600, 2010.

\bibitem[Jaques et~al.(2018)Jaques, Lazaridou, Hughes, Gulcehre, Ortega,
  Strouse, Leibo, and de~Freitas]{jaques2018intrinsic}
Natasha Jaques, Angeliki Lazaridou, Edward Hughes, Caglar Gulcehre, Pedro~A
  Ortega, DJ~Strouse, Joel~Z Leibo, and Nando de~Freitas.
\newblock Intrinsic social motivation via causal influence in multi-agent rl.
\newblock \emph{arXiv preprint arXiv:1810.08647}, 2018.

\bibitem[Jin et~al.(2018)Jin, Allen-Zhu, Bubeck, and Jordan]{jin2018is}
Chi Jin, Zeyuan Allen-Zhu, Sebastien Bubeck, and Michael~I Jordan.
\newblock Is q-learning provably efficient?
\newblock In \emph{Advances in Neural Information Processing Systems}, pp.\
  4863--4873, 2018.

\bibitem[Kingma \& Ba(2014)Kingma and Ba]{Kingma2014Adam}
Diederik Kingma and Jimmy Ba.
\newblock Adam: A method for stochastic optimization.
\newblock \emph{Computer Science}, 2014.

\bibitem[Lowe et~al.(2017)Lowe, Wu, Tamar, Harb, Abbeel, and
  Mordatch]{lowe2017multi}
Ryan Lowe, Yi~Wu, Aviv Tamar, Jean Harb, OpenAI~Pieter Abbeel, and Igor
  Mordatch.
\newblock Multi-agent actor-critic for mixed cooperative-competitive
  environments.
\newblock In \emph{Advances in Neural Information Processing Systems}, pp.\
  6379--6390, 2017.

\bibitem[Mirowski et~al.(2016)Mirowski, Pascanu, Viola, Soyer, Ballard, Banino,
  Denil, Goroshin, Sifre, Kavukcuoglu, et~al.]{mirowski2016learning}
Piotr Mirowski, Razvan Pascanu, Fabio Viola, Hubert Soyer, Andrew~J Ballard,
  Andrea Banino, Misha Denil, Ross Goroshin, Laurent Sifre, Koray Kavukcuoglu,
  et~al.
\newblock Learning to navigate in complex environments.
\newblock \emph{arXiv preprint arXiv:1611.03673}, 2016.

\bibitem[Mnih et~al.(2015)Mnih, Kavukcuoglu, Silver, Rusu, Veness, Bellemare,
  Graves, Riedmiller, Fidjeland, Ostrovski, et~al.]{mnih2015human}
Volodymyr Mnih, Koray Kavukcuoglu, David Silver, Andrei~A Rusu, Joel Veness,
  Marc~G Bellemare, Alex Graves, Martin Riedmiller, Andreas~K Fidjeland, Georg
  Ostrovski, et~al.
\newblock Human-level control through deep reinforcement learning.
\newblock \emph{Nature}, 518\penalty0 (7540):\penalty0 529, 2015.

\bibitem[Mohamed \& Rezende(2015)Mohamed and Rezende]{mohamed2015variational}
Shakir Mohamed and Danilo~Jimenez Rezende.
\newblock Variational information maximisation for intrinsically motivated
  reinforcement learning.
\newblock In \emph{Advances in neural information processing systems}, pp.\
  2125--2133, 2015.

\bibitem[Nguyen et~al.(2018)Nguyen, Kumar, and Lau]{nguyen2018credit}
Duc~Thien Nguyen, Akshat Kumar, and Hoong~Chuin Lau.
\newblock Credit assignment for collective multiagent rl with global rewards.
\newblock In \emph{Advances in Neural Information Processing Systems}, pp.\
  8102--8113, 2018.

\bibitem[Now{\'e} et~al.(2012)Now{\'e}, Vrancx, and De~Hauwere]{nowe2012game}
Ann Now{\'e}, Peter Vrancx, and Yann-Micha{\"e}l De~Hauwere.
\newblock Game theory and multi-agent reinforcement learning.
\newblock In \emph{Reinforcement Learning}, pp.\  441--470. Springer, 2012.

\bibitem[OpenAI(2018)]{OpenAI_dota}
OpenAI.
\newblock Openai five.
\newblock \url{https://blog.openai.com/openai-five/}, 2018.

\bibitem[Osband \& Van~Roy(2016)Osband and Van~Roy]{osband2016lower}
Ian Osband and Benjamin Van~Roy.
\newblock On lower bounds for regret in reinforcement learning.
\newblock \emph{arXiv preprint arXiv:1608.02732}, 2016.

\bibitem[Osband et~al.(2013)Osband, Russo, and Van~Roy]{osband2013more}
Ian Osband, Daniel Russo, and Benjamin Van~Roy.
\newblock (more) efficient reinforcement learning via posterior sampling.
\newblock In \emph{Advances in Neural Information Processing Systems}, pp.\
  3003--3011, 2013.

\bibitem[Ostrovski et~al.(2017)Ostrovski, Bellemare, van~den Oord, and
  Munos]{ostrovski2017count}
Georg Ostrovski, Marc~G Bellemare, A{\"a}ron van~den Oord, and R{\'e}mi Munos.
\newblock Count-based exploration with neural density models.
\newblock In \emph{Proceedings of the 34th International Conference on Machine
  Learning-Volume 70}, pp.\  2721--2730. JMLR. org, 2017.

\bibitem[Oudeyer \& Kaplan(2009)Oudeyer and Kaplan]{oudeyer2009intrinsic}
Pierre-Yves Oudeyer and Frederic Kaplan.
\newblock What is intrinsic motivation? a typology of computational approaches.
\newblock \emph{Frontiers in neurorobotics}, 1:\penalty0 6, 2009.

\bibitem[Oudeyer et~al.(2007)Oudeyer, Kaplan, and Hafner]{oudeyer2007intrinsic}
Pierre-Yves Oudeyer, Frdric Kaplan, and Verena~V Hafner.
\newblock Intrinsic motivation systems for autonomous mental development.
\newblock \emph{IEEE transactions on evolutionary computation}, 11\penalty0
  (2):\penalty0 265--286, 2007.

\bibitem[Pathak et~al.(2017)Pathak, Agrawal, Efros, and
  Darrell]{pathak2017curiosity}
Deepak Pathak, Pulkit Agrawal, Alexei~A Efros, and Trevor Darrell.
\newblock Curiosity-driven exploration by self-supervised prediction.
\newblock In \emph{International Conference on Machine Learning}, pp.\
  2778--2787, 2017.

\bibitem[Pathak et~al.(2018)Pathak, Mahmoudieh, Luo, Agrawal, Chen, Shentu,
  Shelhamer, Malik, Efros, and Darrell]{pathak2018zero}
Deepak Pathak, Parsa Mahmoudieh, Guanghao Luo, Pulkit Agrawal, Dian Chen, Yide
  Shentu, Evan Shelhamer, Jitendra Malik, Alexei~A Efros, and Trevor Darrell.
\newblock Zero-shot visual imitation.
\newblock In \emph{Proceedings of the IEEE Conference on Computer Vision and
  Pattern Recognition Workshops}, pp.\  2050--2053, 2018.

\bibitem[Peysakhovich \& Lerer(2018)Peysakhovich and
  Lerer]{peysakhovich2018prosocial}
Alexander Peysakhovich and Adam Lerer.
\newblock Prosocial learning agents solve generalized stag hunts better than
  selfish ones.
\newblock In \emph{Proceedings of the 17th International Conference on
  Autonomous Agents and MultiAgent Systems}, pp.\  2043--2044. International
  Foundation for Autonomous Agents and Multiagent Systems, 2018.

\bibitem[Rashid et~al.(2018)Rashid, Samvelyan, Witt, Farquhar, Foerster, and
  Whiteson]{rashid2018qmix}
Tabish Rashid, Mikayel Samvelyan, Christian~Schroeder Witt, Gregory Farquhar,
  Jakob Foerster, and Shimon Whiteson.
\newblock Qmix: Monotonic value function factorisation for deep multi-agent
  reinforcement learning.
\newblock In \emph{International Conference on Machine Learning}, pp.\
  4292--4301, 2018.

\bibitem[Riedmiller et~al.(2018)Riedmiller, Hafner, Lampe, Neunert, Degrave,
  Van~de Wiele, Mnih, Heess, and Springenberg]{riedmiller2018learning}
Martin Riedmiller, Roland Hafner, Thomas Lampe, Michael Neunert, Jonas Degrave,
  Tom Van~de Wiele, Volodymyr Mnih, Nicolas Heess, and Jost~Tobias
  Springenberg.
\newblock Learning by playing-solving sparse reward tasks from scratch.
\newblock \emph{arXiv preprint arXiv:1802.10567}, 2018.

\bibitem[Rubin et~al.(2012)Rubin, Shamir, and Tishby]{rubin2012trading}
Jonathan Rubin, Ohad Shamir, and Naftali Tishby.
\newblock Trading value and information in mdps.
\newblock In \emph{Decision Making with Imperfect Decision Makers}, pp.\
  57--74. Springer, 2012.

\bibitem[Salge et~al.(2014)Salge, Glackin, and Polani]{salge2014changing}
Christoph Salge, Cornelius Glackin, and Daniel Polani.
\newblock Changing the environment based on empowerment as intrinsic
  motivation.
\newblock \emph{Entropy}, 16\penalty0 (5):\penalty0 2789--2819, 2014.

\bibitem[Schaul et~al.(2015)Schaul, Quan, Antonoglou, and
  Silver]{schaul2015prioritized}
Tom Schaul, John Quan, Ioannis Antonoglou, and David Silver.
\newblock Prioritized experience replay.
\newblock \emph{arXiv preprint arXiv:1511.05952}, 2015.

\bibitem[Schmidhuber(1991)]{schmidhuber1991possibility}
J{\"u}rgen Schmidhuber.
\newblock A possibility for implementing curiosity and boredom in
  model-building neural controllers.
\newblock In \emph{Proc. of the international conference on simulation of
  adaptive behavior: From animals to animats}, pp.\  222--227, 1991.

\bibitem[Schulman et~al.(2017)Schulman, Wolski, Dhariwal, Radford, and
  Klimov]{schulman2017proximal}
John Schulman, Filip Wolski, Prafulla Dhariwal, Alec Radford, and Oleg Klimov.
\newblock Proximal policy optimization algorithms.
\newblock \emph{arXiv preprint arXiv:1707.06347}, 2017.

\bibitem[Still \& Precup(2012)Still and Precup]{still2012information}
Susanne Still and Doina Precup.
\newblock An information-theoretic approach to curiosity-driven reinforcement
  learning.
\newblock \emph{Theory in Biosciences}, 131\penalty0 (3):\penalty0 139--148,
  2012.

\bibitem[Strens(2000)]{strens2000bayesian}
Malcolm Strens.
\newblock A bayesian framework for reinforcement learning.
\newblock In \emph{ICML}, volume 2000, pp.\  943--950, 2000.

\bibitem[Strouse et~al.(2018)Strouse, Kleiman-Weiner, Tenenbaum, Botvinick, and
  Schwab]{strouse2018learning}
DJ~Strouse, Max Kleiman-Weiner, Josh Tenenbaum, Matt Botvinick, and David~J
  Schwab.
\newblock Learning to share and hide intentions using information
  regularization.
\newblock In \emph{Advances in Neural Information Processing Systems}, pp.\
  10249--10259, 2018.

\bibitem[Sutton et~al.(2000)Sutton, McAllester, Singh, and
  Mansour]{sutton2000policy}
Richard~S Sutton, David~A McAllester, Satinder~P Singh, and Yishay Mansour.
\newblock Policy gradient methods for reinforcement learning with function
  approximation.
\newblock In \emph{Advances in neural information processing systems}, pp.\
  1057--1063, 2000.

\bibitem[Tang et~al.(2017)Tang, Houthooft, Foote, Stooke, Chen, Duan, Schulman,
  DeTurck, and Abbeel]{tang2017exploration}
Haoran Tang, Rein Houthooft, Davis Foote, Adam Stooke, OpenAI~Xi Chen, Yan
  Duan, John Schulman, Filip DeTurck, and Pieter Abbeel.
\newblock \# exploration: A study of count-based exploration for deep
  reinforcement learning.
\newblock In \emph{Advances in neural information processing systems}, pp.\
  2753--2762, 2017.

\bibitem[Wolpert \& Tumer(2002)Wolpert and Tumer]{wolpert2002optimal}
David~H Wolpert and Kagan Tumer.
\newblock Optimal payoff functions for members of collectives.
\newblock In \emph{Modeling complexity in economic and social systems}, pp.\
  355--369. World Scientific, 2002.

\bibitem[Wu et~al.(2018)Wu, Wu, Gkioxari, and Tian]{wu2018building}
Yi~Wu, Yuxin Wu, Georgia Gkioxari, and Yuandong Tian.
\newblock Building generalizable agents with a realistic and rich 3d
  environment.
\newblock \emph{arXiv preprint arXiv:1801.02209}, 2018.

\bibitem[Wu \& Tian(2016)Wu and Tian]{wu2016training}
Yuxin Wu and Yuandong Tian.
\newblock Training agent for first-person shooter game with actor-critic
  curriculum learning.
\newblock \emph{ICLR}, 2016.

\bibitem[Zhang \& Lesser(2011)Zhang and Lesser]{zhang2011coordinated}
Chongjie Zhang and Victor Lesser.
\newblock Coordinated multi-agent reinforcement learning in networked
  distributed pomdps.
\newblock In \emph{Twenty-Fifth AAAI Conference on Artificial Intelligence},
  2011.

\end{thebibliography}
\bibliographystyle{iclr2020_conference}

\newpage
\appendix
\section*{Appendix}

\section{Intrinsic EDTI}\label{appendix:voi_int}
Value of interaction (VoI) captures both transition and reward influence among agents, and it facilitates coordinated exploration by encouraging interactions. VoI contains influence of both intrinsic and extrinsic rewards. Since single-agent literature has studied purely curiosity-driven learning and gets cutting-edge performance~\citep{burda2018large}, it is interesting to investigate the performance of VoI given only intrinsic rewards.

Intuitively, intrinsic VoI distributes individual curiosity among team members and facilitates exploration by encouraging agents to help each other to reach under-explored states. Specifically, we use the following objective:
\begin{equation}
    \begin{aligned}
    J_{\theta_i}[\pi_{i} | \pi_{-i}, p_0] \equiv V^{ext, \bm{\pi}}(\vs_0) + V_i^{int, \bm{\pi}}(\vs_0) + \beta \cdot  VoI_{-i|i}^{int, \bm{\pi}}.
    \end{aligned}
\end{equation}
The corresponding augmented reward is:
\begin{equation}\label{equ:edti_intrinsic_reward}
\hat{r}_1^t = r_t + u_1^t + \beta \left[u_2^t + \gamma \left(1 - \frac{p^-(s_2^{t+1} | s_2^t, a_2^t)}{p(s_2^{t+1}|s_1^t,s_2^t,a_1^t,a_2^t)}\right) V^{int, -}_2(s_1^{t+1}, s_2^{t+1}) \right]
\end{equation}
We use this method (intrinsic EDTI) to train the agents on pass, secret-room, push-box, and island and show the results in Fig.~\ref{fig:intrinsic_edti}.

\subsection{Task Structure}

\section{Mathematical Details}
\subsection{Gradient of Mutual Information} \label{appendix:mi}
To encourage agents to exert influence on transitions of other agents, we optimize mutual information between agent's trajectories. In particular, in the following, we show that term 2 in Eq.~\ref{equ:first_step_of_gradient_of_mi} is always zero.

\begin{eqnarray}
T2 &=& \sum_{\vs,\va,s_2' \in(S, A, S_2)} p^{\bm{\pi}}(\vs,\va,s_2') \nabla_{\theta_1} \log\frac{p(s_2'|\vs,\va)} {p^{\bm{\pi}}(s_2'|s_2,a_2)} \\
&=& -\sum_{\vs,\va,s_2'} p^{\bm{\pi}}(\vs,\va,s_2') \nabla_{\theta_1}\log{p^{\bm{\pi}}(s_2'|s_2,a_2)} \\
&=& -\sum_{\vs,\va,s_2'} p^{\bm{\pi}}(\vs,\va,s_2') \frac{\nabla_{\theta_1}(p^{\bm{\pi}}(s_2'|s_2,a_2))}{p^{\bm{\pi}}(s_2'|s_2,a_2)} \\
&=& -\sum_{\vs,\va,s_2'} \frac{p^{\bm{\pi}}(\vs,\va,s_2')}{p^{\bm{\pi}}(s_2'|s_2,a_2)} \nabla_{\theta_1}\left(\sum_{s_1^*, a_1^*}p(s_2'|s_2,a_2, s_1^*, a_1^*)p(s_1^* | s_2, a_2)\pi_{\theta_1}(a_1^*|s_1^*)\right) \\
&=& -\sum_{\vs,\va,s_2'} \frac{p^{\bm{\pi}}(\vs,\va,s_2')}{p^{\bm{\pi}}(s_2'|s_2,a_2)} \sum_{s_1^*, a_1^*}p(s_2'|s_2,a_2, s_1^*, a_1^*)p(s_1^* | s_2, a_2)\nabla_{\theta_1}\pi_{\theta_1}(a_1^*|s_1^*)\\
&=& -\sum_{s_2,a_2,s_2'} \frac{p^{\bm{\pi}}(s_2,a_2,s_2')}{p^{\bm{\pi}}(s_2'|s_2,a_2)} \sum_{s_1^*, a_1^*}p(s_2'|s_2,a_2, s_1^*, a_1^*)p(s_1^* | s_2, a_2)\nabla_{\theta_1}\pi_{\theta_1}(a_1^*|s_1^*)  \\
&=& -\sum_{s_2,a_2,s_2'} p^{\bm{\pi}}(s_2,a_2) \sum_{s_1^*, a_1^*}p(s_2'|s_2,a_2, s_1^*, a_1^*)p(s_1^* | s_2, a_2)\nabla_{\theta_1}\pi_{\theta_1}(a_1^*|s_1^*)  \\
&=& -\sum_{s_2,a_2} p^{\bm{\pi}}(s_2,a_2) \sum_{s_1^*, a_1^*}p(s_1^* | s_2, a_2)\nabla_{\theta_1}\pi_{\theta_1}(a_1^*|s_1^*) \sum_{s_2'}p(s_2'|s_2,a_2, s_1^*, a_1^*) \\
&=& -\sum_{s_2,a_2} p^{\bm{\pi}}(s_2,a_2) \sum_{s_1^*, a_1^*}p(s_1^* | s_2, a_2)\nabla_{\theta_1}\pi_{\theta_1}(a_1^*|s_1^*) \underbrace{\sum_{s_2'}p(s_2'|s_2,a_2, s_1^*, a_1^*)}_{= 1} \\
&=& -\sum_{s_2,a_2} p^{\bm{\pi}}(s_2,a_2) \sum_{s_1^*}p(s_1^* | s_2, a_2)\nabla_{\theta_1}\sum_{a_1^*}\pi_{\theta_1}(a_1^*|s_1^*) \\
&=& -\sum_{s_2,a_2} p^{\bm{\pi}}(s_2,a_2) \sum_{s_1^*}p(s_1^* | s_2, a_2)\nabla_{\theta_1}1 \\
&=& 0
\end{eqnarray}

\begin{figure}
    \centering
    \includegraphics[width=0.49\linewidth]{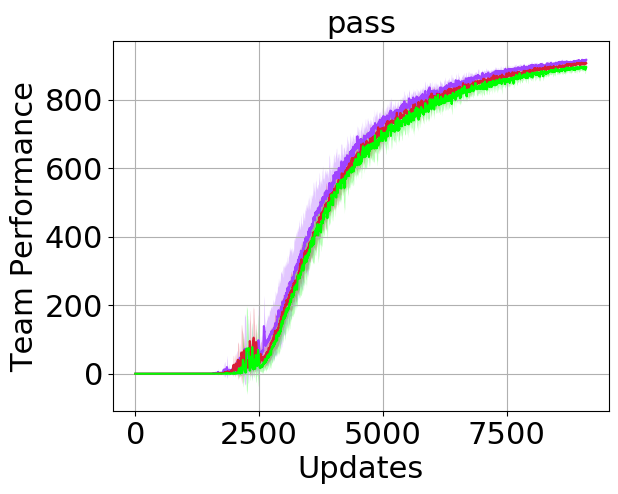}\label{fig:voi_int_pass}\hfill
    \includegraphics[width=0.49\linewidth]{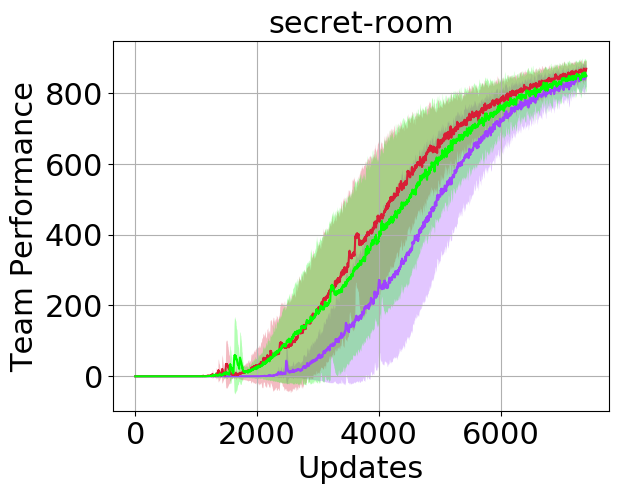}\label{fig:voi_int_threepass}\\
    \includegraphics[width=0.49\linewidth]{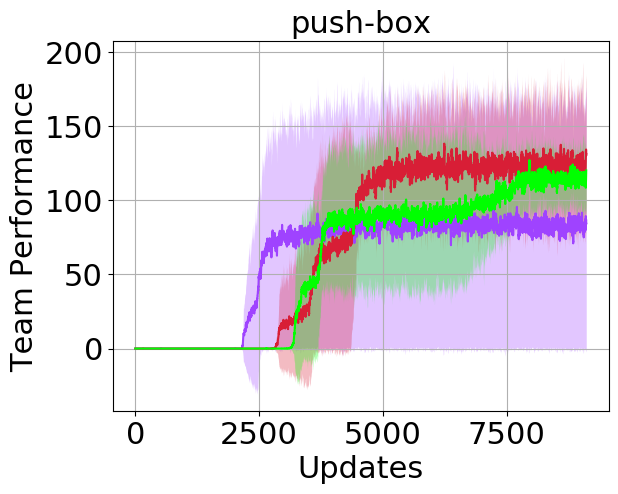}\label{fig:voi_int_pushball}\hfill
    \includegraphics[width=0.49\linewidth]{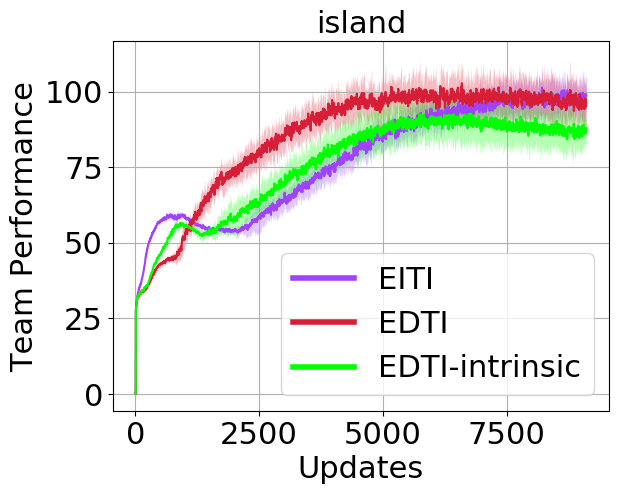}\label{fig:voi_int_island}
    \caption{Performance of intrinsic EDTI in comparison with EITI and EDTI on pass, secret-room, push-box, and island.}
    \label{fig:intrinsic_edti}
\end{figure}

\subsection{Definition of \emph{Value of Interaction}}\label{appendix:voi_definition_proof}
To capture both transition and reward interactions between agents and thereby achieve intrinsic reward distribution, we propose a decision-theoretic measure called \emph{Value of Interaction}. We start from 2-agent cases and the following theorem gives the definition of $VoI_{2|1}$ in the form of an expectation over trajectories, which is especially helpful in the derivation of the EDTI policy gradient update shown Eq.~\ref{equ:edti_gradient}.
\begin{restatable}{Theorem}{thovoi}
Value of Interaction of agent 1 on agent 2 is:
    \begin{equation}
    \begin{aligned}
        VoI_{2|1}^{\bm{\pi}}(S_2';S_1,A_1|S_2,A_2) = \mathbb{E}_\tau\left[\tilde{r}_2(\bm{s}, \bm{a}) - \tilde{r}_2^{\bm{\pi}}(s_2, a_2)
        + \gamma \left(1 - \frac{p^{\bm{\pi}}(s_2' | s_2, a_2)}{p(s_2' |\bm{s}, \bm{a})}\right) V_2^{\bm{\pi}}(\bm{s}')\right],
    \end{aligned}
    \end{equation}
        where $\tilde{r}_2^{\bm{\pi}}(s_2, a_2)$ is the counterfactual immediate reward.
\label{tho:voi}\end{restatable}

$VoI_{2|1}$ can be defined similarly. To lighten notation in the proof, we define
\begin{eqnarray}
    V_2^{\bm{\pi}}(s_2' | s_1, s_2, a_1, a_2) = \sum_{s_1'} p(s_1' | s_1, s_2, a_1, a_2, s_2') V_2^{\bm{\pi}}(s_1', s_2') \\
	\tilde{r}_2^{\bm{\pi}}(s_2, a_2) = \sum_{s_1^*, a_1^*} p^{\bm{\pi}}(s_1^*, a_1^* | s_2,  a_2) \tilde{r}_2(s_1^*, s_2,  a_1^*, a_2), \\
	V_2^{\bm{\pi}, *}(s_2' | s_2,  a_2) = \sum_{s_1^*, a_1^*} p^{\bm{\pi}}(s_1^*, a_1^* | s_2,  a_2) \sum_{s_1'} p(s_1' | s_1^*, s_2, a_1^*, a_2, s_2') V_2^{\bm{\pi}}(s_1', s_2').
\end{eqnarray}
We first prove Lemma \ref{lemma:1}, which is used in the proof of Theorem \ref{tho:voi}.
\begin{lemma}\label{lemma:1}
	\begin{eqnarray}
	&& \sum_{s_1, s_2, a_1, a_2}p^{\bm{\pi}}(s_1, s_2,  a_1, a_2) \gamma \sum_{s_2'} p(s_2' |s_1, s_2, a_1, a_2) V_2^{\bm{\pi}}(s_2' | s_2, a_2) \\
	&=& \sum_{s_1, s_2, a_1, a_2}p^{\bm{\pi}}(s_1, s_2,  a_1, a_2) \gamma \sum_{s_1', s_2'} T(s_1', s_2' |s_1, s_2,  a_1, a_2) \cdot \frac{p^{\bm{\pi}}(s_2' | s_2,  a_2)}{p(s_2' | s_1, s_2,  a_1, a_2)} V_2^{\bm{\pi}}(s_1', s_2').\nonumber
	\end{eqnarray}
\end{lemma}
\begin{proof}
	\begin{eqnarray}
	&& \sum_{s_1, s_2, a_1, a_2}p^{\bm{\pi}}(s_1, s_2,  a_1, a_2) \gamma \sum_{s_2'} p(s_2' |s_1, s_2, a_1, a_2) V_2^{\bm{\pi}}(s_2' | s_2, a_2) \\
	&=& \sum_{s_1, s_2, a_1, a_2} p^{\bm{\pi}}(s_1, s_2, a_1, a_2) \gamma \sum_{s_2'} p(s_2' |s_1, s_2, a_1, a_2) \cdot \\
	&& \sum_{s_1^*, a_1^*} p^{\bm{\pi}}(s_1^*, a_1^* | s_2,  a_2) \sum_{s_1'} p(s_1' | s_1^*, s_2, a_1^*, a_2, s_2') V_2^{\bm{\pi}}(s_1', s_2') \\
	&=& \sum_{s_1, s_2, a_1, a_2}p^{\bm{\pi}}(s_1, s_2,  a_1, a_2) \gamma \sum_{s_2'} p(s_2' |s_1, s_2, a_1, a_2) \cdot \\
	&& \sum_{s_1^*, a_1^*} p^{\bm{\pi}}(s_1^*, a_1^* | s_2,  a_2) \sum_{s_1'} \frac{T(s_1', s_2' | s_1^*, s_2,  a_1^*, a_2)}{p(s_2' | s_1^*, s_2, a_1^*, a_2)} V_2^{\bm{\pi}}(s_1', s_2') \\
	&=& \sum_{s_1, s_2,  a_1, a_2}p^{\bm{\pi}}(s_1, s_2, a_1, a_2) \gamma \sum_{s_1', s_2'} \frac{T(s_1', s_2' | s_1, s_2,  a_1, a_2)}{p(s_2' | s_1, s_2,  a_1, a_2)} \cdot \\
	&& V_2^{\bm{\pi}}(s_1', s_2') \sum_{s_1^*, a_1^*} p^{\bm{\pi}}(s_1^*, a_1^* | s_2,  a_2) p(s_2' | s_1^*, s_2, a_1^*, a_2) \\
	&=& \sum_{s_1, s_2,  a_1, a_2}p^{\bm{\pi}}(s_1, s_2, a_1, a_2) \gamma \sum_{s_1', s_2'} T(s_1', s_2' | s_1, s_2,  a_1, a_2)\cdot \\
	&& \frac{p^{\bm{\pi}}(s_2' | s_2,  a_2)}{p(s_2' | s_1, s_2,  a_1, a_2)} V_2^{\bm{\pi}}(s_1', s_2'). 
	\end{eqnarray}
\end{proof}
We now give the proof of Theorem~\ref{tho:voi}:
\begin{proof}
    \begin{eqnarray}
    	&& VoI_{2|1}^{\bm{\pi}}(S_2';S_1,A_1|S_2,A_2) \\
    	&=& \sum_{\bm{s},\bm{a},s_2'\in(S,A,S_2)} p^{\bm{\pi}}(\bm{s},\bm{a},s_2') \left[Q_2^{\bm{\pi}}(\bm{s},\bm{a},s_2') - Q_{2|1}^{\bm{\pi}, *}(s_2,a_2,s_2')\right] \\
    	&=& \sum_{s_1, s_2, a_1, a_2} p^{\bm{\pi}}(s_1, s_2, a_1, a_2) (\tilde{r}_2(s_1, s_2, a_1, a_2) - \tilde{r}_2^{\bm{\pi}}(s_2, a_2)  + \\
    	&& \gamma \sum_{s_2'} p(s_2' | s_1, s_2, a_1, a_2) (V_2^{\bm{\pi}}(s_2' | s_1, s_2, a_1, a_2) - V_2^{\bm{\pi},*}(s_2' | s_2, a_2)) \\
    	&=& \sum_{s_1, s_2, a_1, a_2} p^{\bm{\pi}}(s_1, s_2, a_1, a_2) (\tilde{r}_2(s_1, s_2, a_1, a_2) - \tilde{r}_2^{\bm{\pi}}(s_2, a_2)  + \\
    	&& \gamma  \sum_{s_1', s_2'} T(s_1', s_2' |s_1, s_2, a_1, a_2) (1 - \frac{p^{\bm{\pi}}(s_2' | s_2, a_2)}{p(s_2' | s_1, s_2, a_1, a_2)}) V_2^{\bm{\pi}}(s_1', s_2')) \textbf{ (Lemma 1)} \\
    	&=& \mathbb{E}_\tau\left[\tilde{r}_2(\vs, \va) - \tilde{r}_2^{\bm{\pi}}(s_2, a_2)
            + \gamma \left(1 - \frac{p^{\bm{\pi}}(s_2' | s_2, a_2)}{p(s_2' |\vs, \va)}\right) V_2^{\bm{\pi}}(\vs')\right].
    \end{eqnarray}
\end{proof}

\subsection{Calculating Gradient of VoI}\label{appendix:gradient_of_VoI}
In order to optimize $VoI$ with respect to the parameters of agent policy, in Sec.~\ref{sec:optimize_VoI}, we propose to use target function and get:
\begin{equation}
\begin{aligned}
        \nabla_{\theta_1}VoI_{2|1}^{\bm{\pi}}(S_2';S_1,A_1|S_2,A_2) \approx \sum_{\vs,\va\in (S,A)} &\left(\nabla_{\theta_1} p^{\bm{\pi}}(\vs,\va)\right)
        \large[\tilde{r}_2(\vs,\va) - \tilde{r}^-_2(s_2, a_2) + \\ &\gamma  \sum_{\vs'} T(\vs'|\vs, \va) \left(1 - \frac{p^-(s_2' | s_2, a_2)}{p(s_2'|\vs, \va)}\right)V^-_2(s_1', s_2')\large].
\end{aligned}
\end{equation}
We prove that $\sum_{\vs, \va} \left(\nabla_{\theta_1} p^{\bm{\pi}}(\vs, \va)\right) \tilde{r}^-_2(s_2, a_2)$ is $0$ in the following lemma.
\begin{lemma}\label{lemma:2}
	\begin{eqnarray}
		\sum_{s_1, s_2, a_1, a_2} \left(\nabla_{\theta_1} p^{\bm{\pi}}(s_1, s_2, a_1, a_2)\right) \tilde{r}^-_2(s_2, a_2) = 0.
	\end{eqnarray}
\end{lemma}
\begin{proof}
	Similar to the way that policy gradient theorem was proved by \cite{sutton2000policy},
	\begin{eqnarray}
		&& \sum_{s_1, s_2, a_1, a_2} \left(\nabla_{\theta_1} p^{\bm{\pi}}(s_1, s_2, a_1, a_2)\right) \tilde{r}^-_2(s_2, a_2) \\
		&=& \sum_{s_2, a_2} p^{\bm{\pi}}(s_2, a_2) \sum_{s_1, a_1} \left(\nabla_{\theta_1} p^{\bm{\pi}}_1(s_1, a_1 | s_2) \right) \tilde{r}^-_2(s_2, a_2) \\
		&=& \sum_{s_2, a_2} p^{\bm{\pi}}(s_2, a_2) \sum_{s_1, a_1} \left(\nabla_{\theta_1} p^{\bm{\pi}}_1(s_1, a_1 | s_2) \right) \sum_{s_1^*, a_1^*} p^-(s_1^*, a_1^* | s_2, a_2) \tilde{r}_2(s_1^*, s_2, a_1^*,  a_2) \\
		&=& \sum_{s_2, a_2} p^{\bm{\pi}}(s_2, a_2) \sum_{s_1^*, a_1^*} p^-(s_1^*, a_1^* | s_2, a_2) \tilde{r}_2(s_1^*, s_2, a_1^*, a_2) \sum_{s_1, a_1} \left(\nabla_{\theta_1} p^{\bm{\pi}}_1(s_1, a_1 | s_2) \right) \\
		&=& \sum_{s_2, a_2} p^{\bm{\pi}}(s_2, a_2) \sum_{s_1^*, a_1^*} p^-(s_1^*, a_1^* | s_2, a_2) \tilde{r}_2(s_1^*, s_2, a_1^*, a_2) \left(\nabla_{\theta_1} 1 \right) \\
		&=& 0.
	\end{eqnarray}
\end{proof}

\subsection{Immediate Reward Influence}\label{appendix:r_influence}
Similar to MI and $VoI$, we can define influence of agent 1 on the immediate rewards of agent 2 as:
\begin{equation}
\begin{aligned}\label{equ:ri}
        RI_{2|1}^{\bm{\pi}}(S_2';S_1,A_1|S_2,A_2) = \sum_{\vs,\va\in (S,A)} p^{\bm{\pi}}(\vs,\va)
        \large[\tilde{r}_2(\vs,\va) - \tilde{r}_2(s_2, a_2)\large].
\end{aligned}
\end{equation}
Use Lemma~\ref{lemma:2}, we can get:
\begin{equation}\label{equ:ri_gradient}
\begin{aligned}
        \nabla_{\theta_1} RI_{2|1}^{\bm{\pi}}(S_2';S_1,A_1|S_2,A_2) = \sum_{\vs,\va\in (S,A)} \nabla_{\theta_1}(p^{\bm{\pi}}(\vs,\va))
        \tilde{r}_2(\vs,\va).
\end{aligned}
\end{equation}
Now we have
\begin{equation}\label{equ:r_influence_gradient}
    \nabla_{\theta_1} J_{\theta_1}(t) \approx \left(\hat{R}_1^t-\hat{V}_1^{\bm{\pi}}(s_t)\right)\nabla_{\theta_1}\log\pi_{\theta_1}(a_1^t|s_1^t),
\end{equation}
where $\hat{V}_1^{\bm{\pi}}(s_t)$ is an augmented value function of $\hat{R}_1^t=\sum_{t'=t}^h\hat{r}_1^{t'}$ and
\begin{equation}\label{equ:r_influence_reward}
\hat{r}_1^t = r^t + u_1^t + \beta u_2^t.
\end{equation}

\section{Estimation of Conditional Probabilities}\label{appendix:p_d_p}
To quantify interdependence among exploration processes of agents, we use mutual information and value of interaction. Calculations of MI and VoI need estimation of  $p(s_2'|s_2,a_2)$ and $p(s_2'|\bm{s},\bm{a})$. In practice, we track the empirical frequencies $p_{emp}(s_2'|s_2,a_2)$ and $p_{emp}(s_2'|\bm{s},\bm{a})$ and substitute them for the corresponding terms in Eq.~\ref{equ:eiti_reward} and~\ref{equ:edti_reward}. % Our method requires to estimate the probability distribution $p(s_2^{t+1}|s_2^t,a_2^t)$ and $p(s_2^{t+1}|s_1^t,s_2^t,a_1^t,a_2^t)$ empirically. 

Estimating $p_{emp}(s_2'|s_2, a_2)$ and $p_{emp}(s_2'|\vs, \va)$ is one obstacle to the scalability of our method, we now discuss how to solve this problem. When the state and action space is small, we can use hash table to implement Monte Carlo method (MC) for estimating the distributions accurately. In the MC sampling, we count from the samples the state frequencies $p(s_2'|s_2,a_2) \equiv \frac{N(s_2', s_2,a_2)}{N(s_2,a_2)}$ and $p(s_2'|\vs, \va) \equiv \frac{N(s_2', s_1,s_2,a_1,a_2)}{N(s_1,s_2,a_1,a_2)}$, where $N(\cdot)$ is the number of times each state-action pair was visited during the learning process. When the problem space becomes large, MC consumes large memory in practice. As an alternative, we adopt variational inference~\citep{fox2012tutorial} to learn variational distributions $q_{\xi_1}(s_2'|s_2,a_2)$ and $q_{\xi_2}(s_2'|\vs, \va)$, parameterized via neural networks with parameters $\xi_1$ and $\xi_2$, by optimizing the evidence lower bound. In Fig.~\ref{fig:vi_pass}, we show the performance of EDTI estimated using variational inference and the changing of associated EDTI rewards on pass during 9000 PPO updates. Variational inference introduces some noise in EDTI rewards estimation and thus requires slightly more steps to learn the true probability and the strategy. However, estimating using MC sampling consumes 1.6G memory to save the hash table with 100M items each agent while variational inference needs a three-layer fully connected network with 74800 parameters occupying about 0.60M memory. This results highlights the feasibility of estimating EITI and EDTI rewards using variational inference in problem with large state-action space.% We clip the associated quotients into range $[0,1]$ to stabilize training.
\begin{figure}
    \centering
    \includegraphics[height=0.28\linewidth]{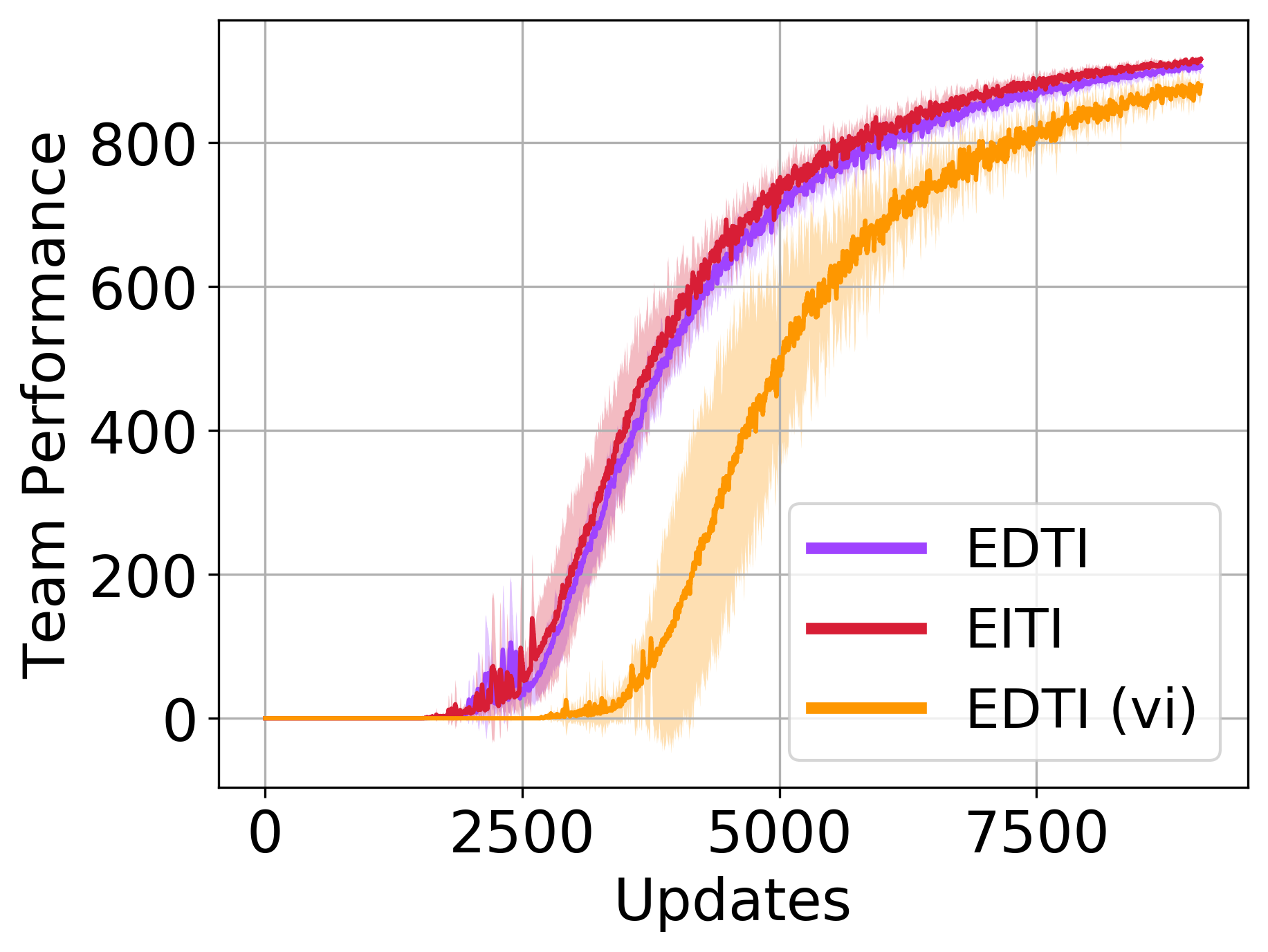}\label{fig:vi_pass_lr}\hfill
    \includegraphics[height=0.28\linewidth]{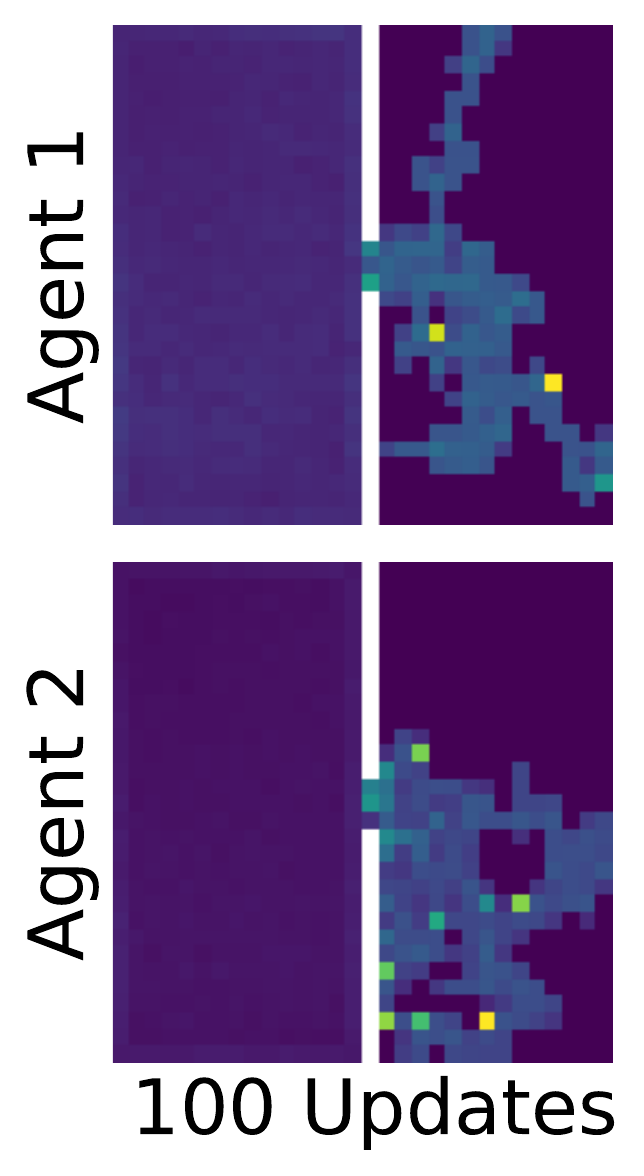}\label{fig:vi_pass_100}\hfill
    \includegraphics[height=0.28\linewidth]{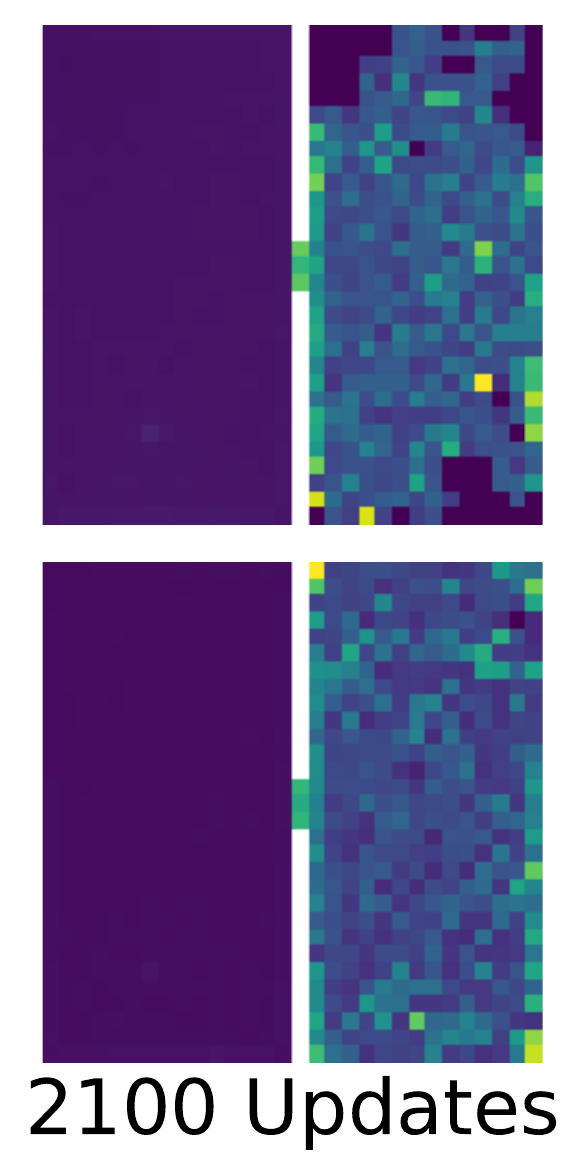}\label{fig:vi_pass_2100}\hfill
    \includegraphics[height=0.28\linewidth]{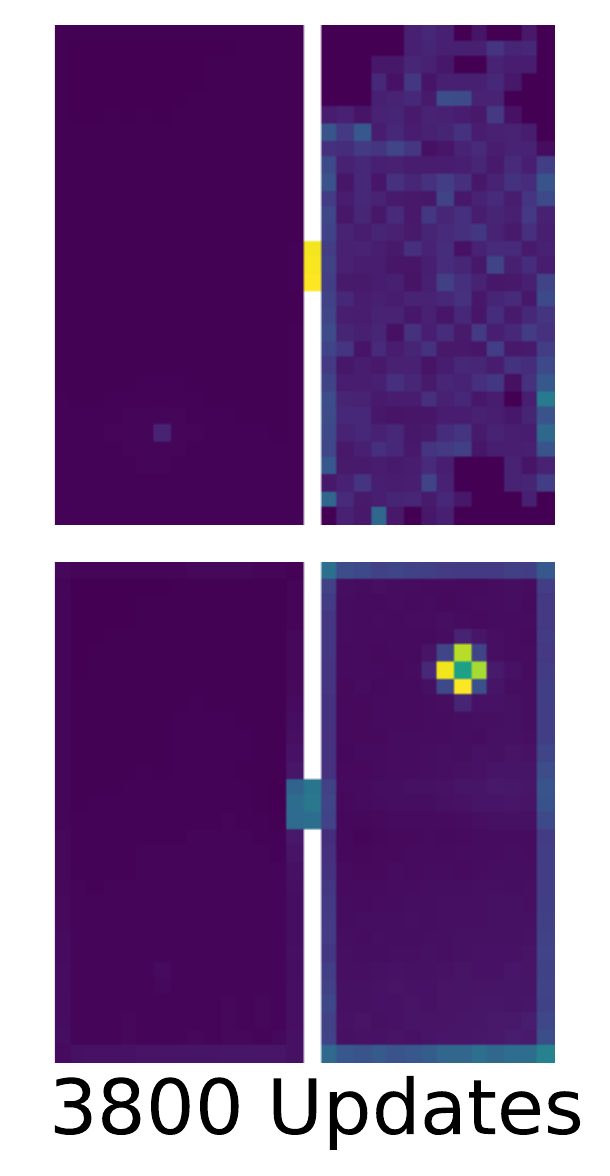}\label{fig:vi_pass_3800}\hfill
    \includegraphics[height=0.28\linewidth]{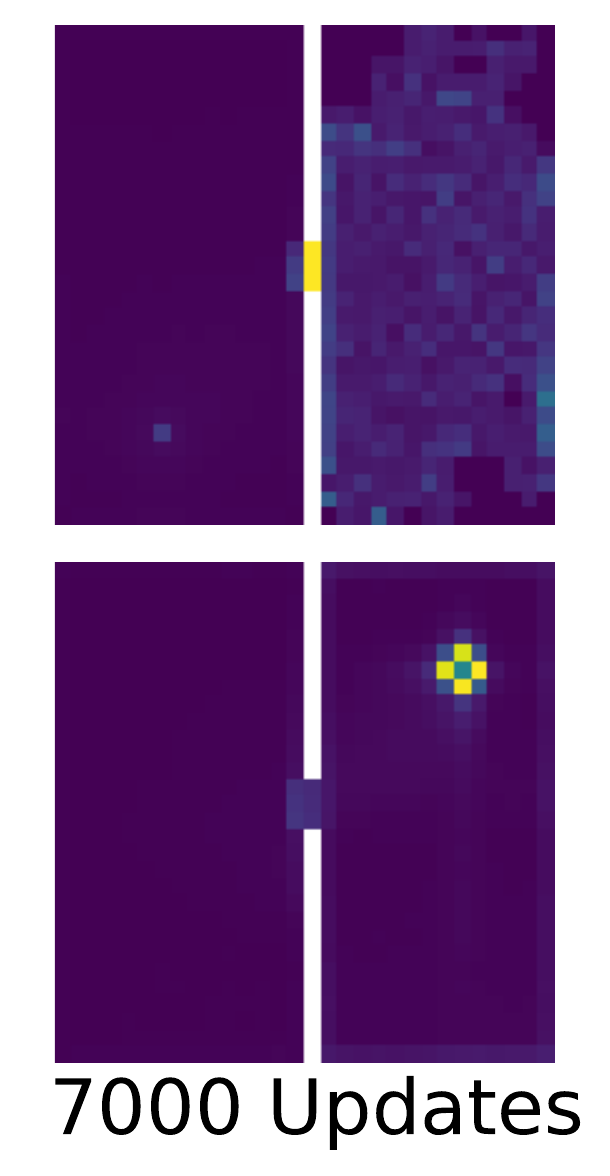}\label{fig:vi_pass_7000}
    \includegraphics[height=0.28\linewidth]{figvi/vi_pass_colorbar.pdf}\label{fig:vi_pass_colorbar}\hfill
    \caption{Left: Performance of EDTI (vi) (EDIT estimated using variational inference) compared with EITI and EDTI estimated using MC sampling. Others: Development of EDTI (vi) rewards during exploration process. Top row: EDTI (vi) rewards of agent 1; bottom row: EDTI (vi) rewards of agent 2.}
    \label{fig:vi_pass}
\end{figure}

\section{Implementation Details}\label{appendix:implementation_details}
\subsection{Network Architecture, Hyperparameters, and Infrastructure}
We base our framework on OpenAI implementation of PPO2~\citep{baselines} and use its default parameters to carry out all the experiments. We train our models on an NVIDIA RTX 2080TI GPU using experience sampled from 32 parallel environments. We use visitation count to calculate the intrinsic reward, for its provable effectiveness~\citep{azar2017minimax, jin2018is}. For all our methods and baselines, we use $\eta/\sqrt{N(s)}$ as the exploration bonus for $N(s)$-th visit to state $s$. Specific values of $\eta$ and scaling weights can be found in Table~\ref{tab:weights}.
% Equations:
% $$\text{Reward of EITI} = r + u_i + \beta_{\text{T}} T_1,$$
% $$\text{Reward of EDTI} = r + u_i + \beta_{\text{int}} (u_{-i} + \gamma T_2 V_{-i}^{\text{int}}) + \beta_{\text{ext}} T_2 V_{-i}^{\text{ext}},$$
% $$\text{Reward of r\_influence} = r + u_i + \beta_{\text{r}} u_{-i},$$
% $$\text{Reward of plusV} = r + u_i + \beta_{\text{int of plusV}} V_{-i}^{\text{int}} + \beta_{\text{ext of plusV}} V_{-i}^{\text{ext}},$$
% and
% $$\text{Reward of Q-Q} = r + u_i + \beta_{\text{int}} \Delta Q_{-i}^{\text{int}} + \beta_{\text{ext}} \Delta Q_{-i}^{\text{ext}},$$
% where $u_i = \frac{\eta}{\sqrt{N(s)}}$.

\begin{table}[t]
    \centering
    \caption{The scaling weights for different intrinsic reward terms in various tasks. $\beta_{\text{T}}$ is the weight of term $T_1$ (see Table~\ref{tab:baselines}). $\beta_{\text{int}}$ and $\beta_{\text{ext}}$ are scaling factors to combine $r$ and $u_i$ in $\tilde{r}$. $u_{\shortn i}$ in r\_influence is scaled by $\beta_{\text{r}}$ while $V^{int}_{\shortn i}$ and $V^{ext}_{\shortn i}$ in plusV are respectively scaled by $\beta^{\text{plusV}}_{\text{int}}$ and $\beta^{\text{plusV}}_{\text{ext}}$.}
    \label{tab:weights}
	\begin{tabular}{cccccccc}
	    \toprule
		Task &  $\eta$ & $\beta_{\text{T}}$ & $\beta_{\text{int}}$ & $\beta_{\text{ext}}$ & $\beta_{\text{r}}$ & $\beta^{\text{plusV}}_{\text{int}}$ & $\beta^{\text{plusV}}_{\text{ext}}$ \\ 
		\cmidrule(lr){1-8}
		Pass &  10. & 10 & 1. & 0.1 & 1. & 0.1 & 0.01 \\ 
        % \cmidrule(lr){1-8}
		Secret-room & 10. & 10 & 1. & 0.1 & --- & --- & --- \\ 
        % \cmidrule(lr){1-8}
		Push-ball &  1. & 100. & 100. & 0.1 & 0.1 & 0.1 & 0.01 \\ 
        % \cmidrule(lr){1-8}
		Island &  1. & 10 & 10. & 0.5 & 0.1 & 0.1 & 0.01 \\ 
        % \cmidrule(lr){1-8}
		Large-island & 1. & 10 & 1. & 0.1 & 0.1 & 0.1 & 0.01 \\ 
		\toprule
	\end{tabular}
\end{table}

As for variational inference, the inference network is a 3-layer fully-connected network coupled with a 64-dimensional reparameterization estimator. ReLU is used as the activation function for the first two layers and the sum of negative log-likelihood and negative Evidence Lower Bound is used as loss. We use Adam optimizer~\citep{Kingma2014Adam} with learning rate $1 \times 10^{\shortn 3}$ and batchsize 2048. To speed up the learning of variational distributions estimation, we equip the learning with proportional prioritized experience replay \citep{schaul2015prioritized}.

\subsection{Task Structure}
\begin{figure}
    \centering
    \includegraphics[height=0.3\linewidth]{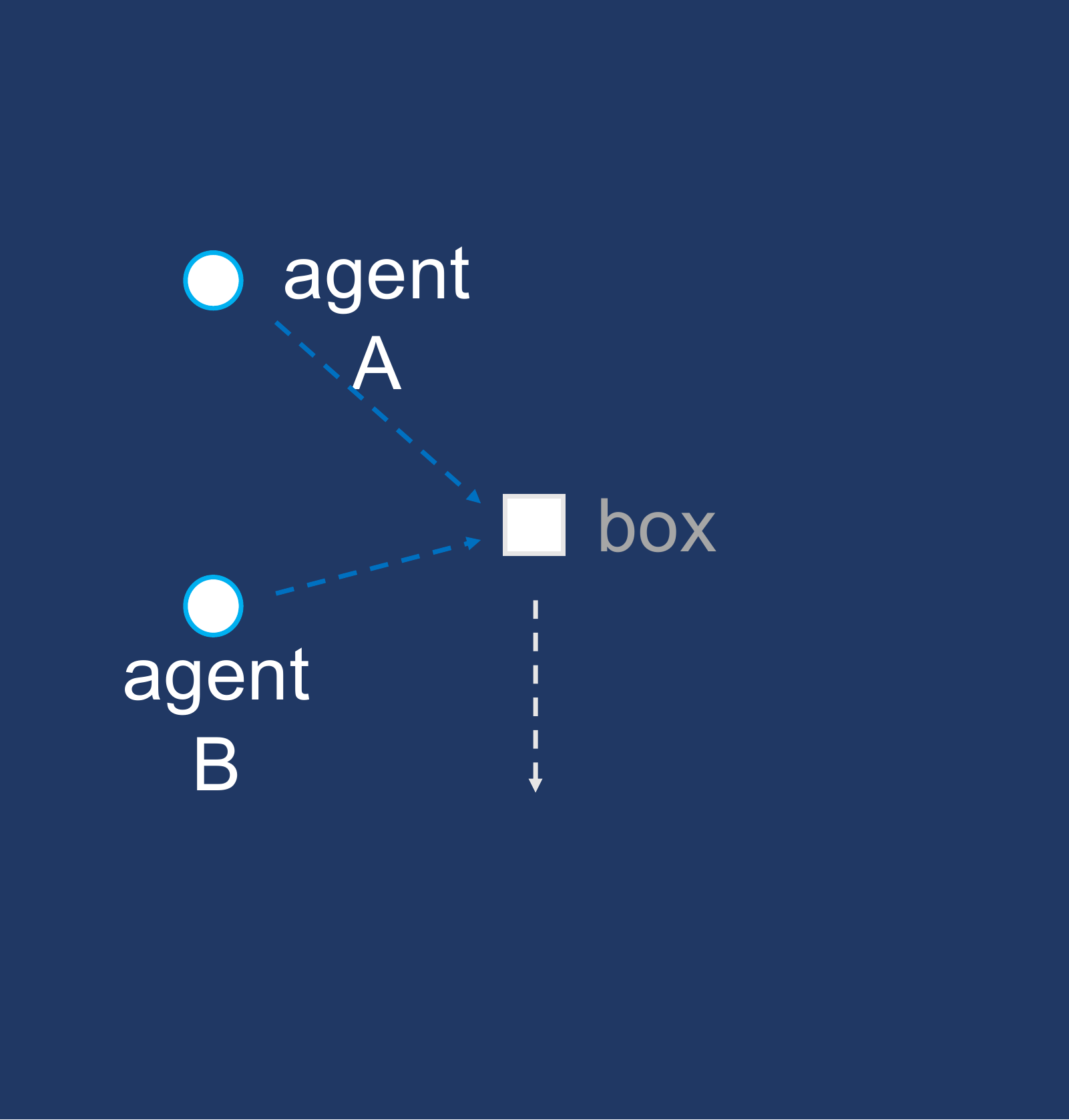}\label{fig:env_push_box}\hfill
    \includegraphics[height=0.3\linewidth]{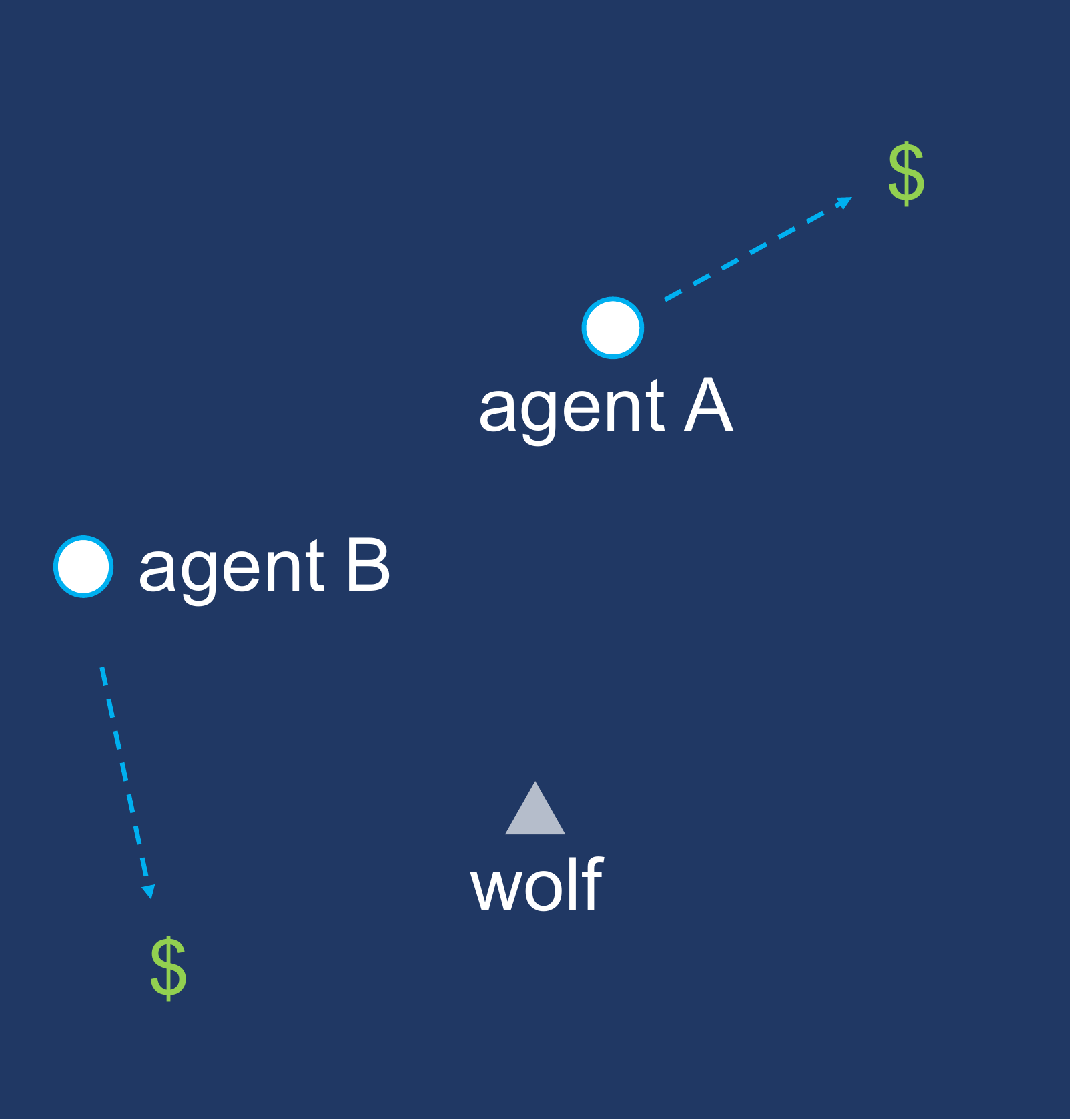}\label{fig:env_island}\hfill
    \includegraphics[height=0.3\linewidth]{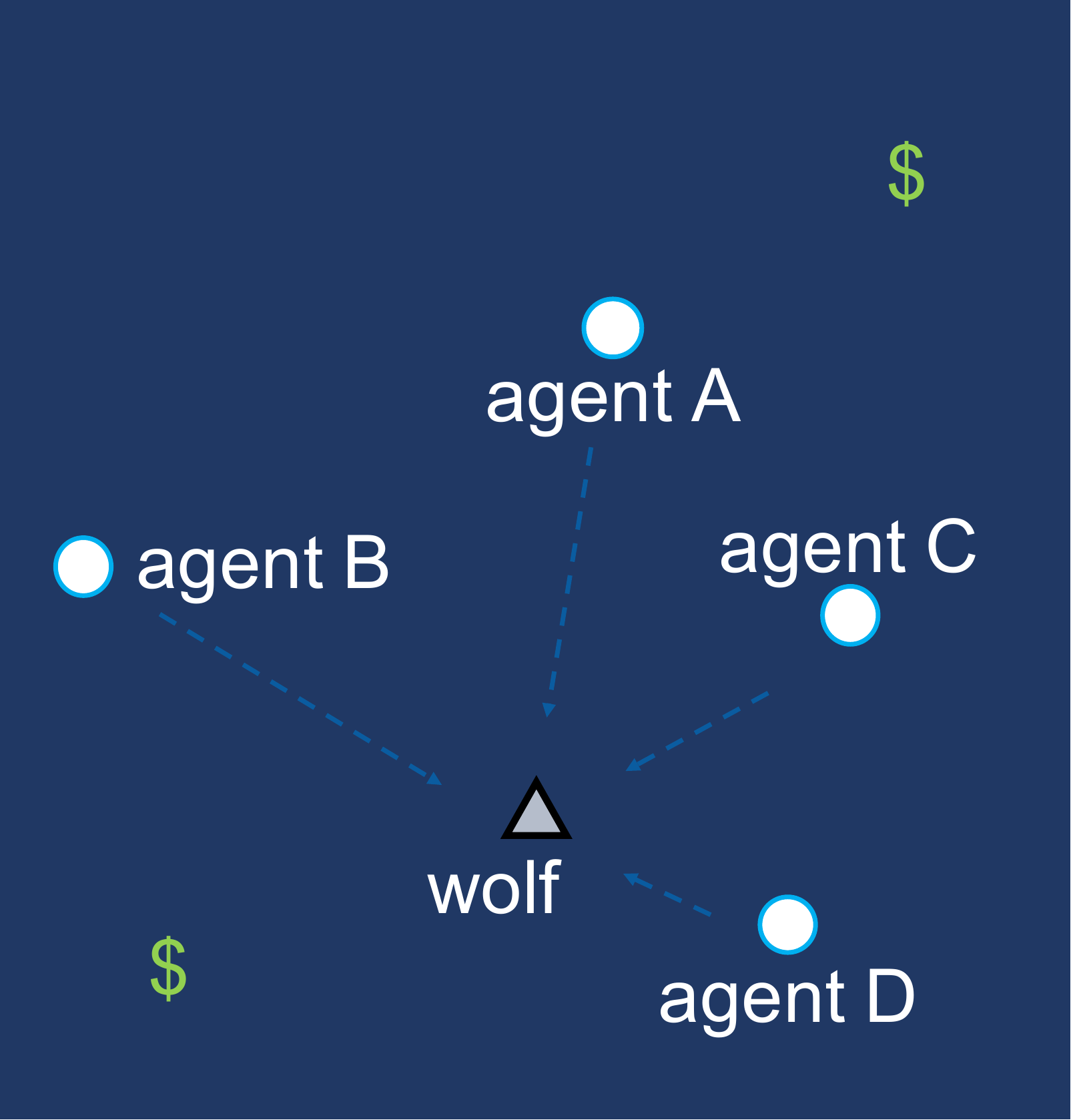}\label{fig:env_large_island}
    \caption{Task \textbf{push-box}, \textbf{island}, and \textbf{large-island}}
    \label{fig:more_env}
\end{figure}
In this section, we describe detailed settings of our task, including task specification and content of $s_i$. 

\textbf{Pass:} There are two agents and two switches to open the door in a $30 \times 30$ grid. Only when at least one of the switches are occupied will the door open. The agents need navigate from left to right and the team reward, which is 1000, is only provided when all agents reach the target zone. Agents can observe the position of another agents.

\textbf{Secret-room:} This is an extension of the pass task with 4 rooms and 4 switches locating in different rooms. The size of the grid is $25\times 25$. When the left switch is occupied, all the three doors are open. And the three switches in each room on the right only control the door of its room. The agents need to navigate towards the desired room (in light red of Fig. \ref{fig:didactic_examples} middle) to achieve the extrinsic team reward 1000. Agents can observe the position of the other agents.

\textbf{Push-ball:} There are two agents and one box in a $15 \times 15$ grid. Agents need to push the box to the wall. However, the box is so heavy that only when two agents push it in the same direction at the same time can it be moved a grid. The only team reward, 1000, is given when the box is placed right against the wall. Agents can observe the coordinates of their teammate and the location of the box.

\textbf{Island:} A group of two agents are hunting for treasure on an island. However, a random walking beast may attack the agents when they are too close. The agents can also attack the beast within their attack range. This hurt doubles when more than one agent attack at the same time. Each agent has a maximum health of 5 and will lose $1/n$ health per step when there are $n$ agents within the attack range of the beast. Island is a modified version of the classic coordination scenario \emph{stag-hunt} with local optimal, because finding each treasure (9 in total) will trigger a team reward of 10 but catching the beast gives a higher team reward of 300. Agents can observe the position and health of each other, and the coordinates of the beast. Fig.~\ref{fig:island_details} shows the development of the probability of catching the beast and the averaged number of treasures found in an episode during 9000 PPO updates.

\textbf{Large-island: } Settings are similar to that of island but with more agents (4), more treasures (16), and a beast with more energy (16) and a higher reward (600) for being caught.

The horizon of one episode is set to 300 timesteps in all these tasks.

\begin{figure}
    \centering
    \includegraphics[height=0.3\linewidth]{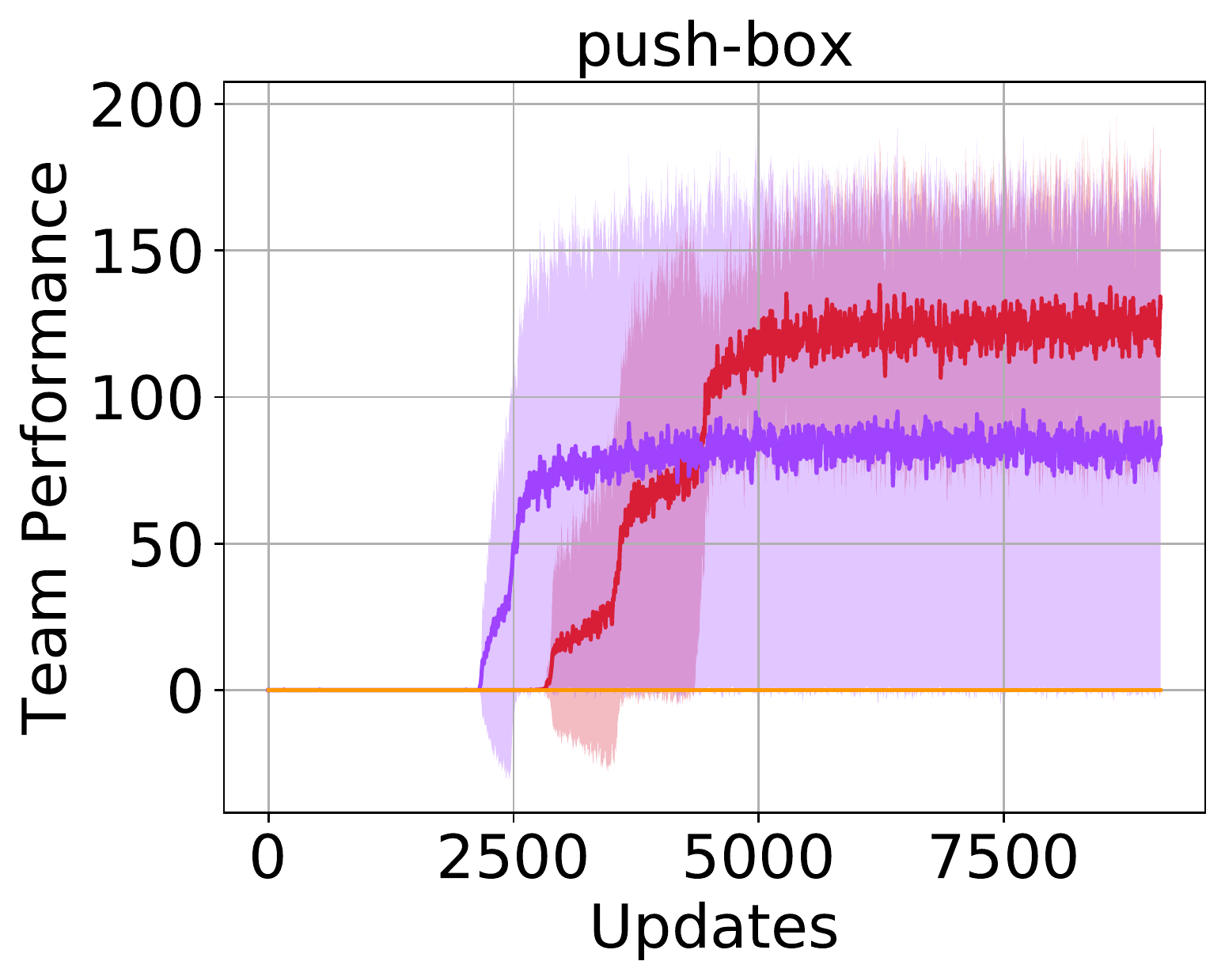}\hfill
    \includegraphics[height=0.3\linewidth]{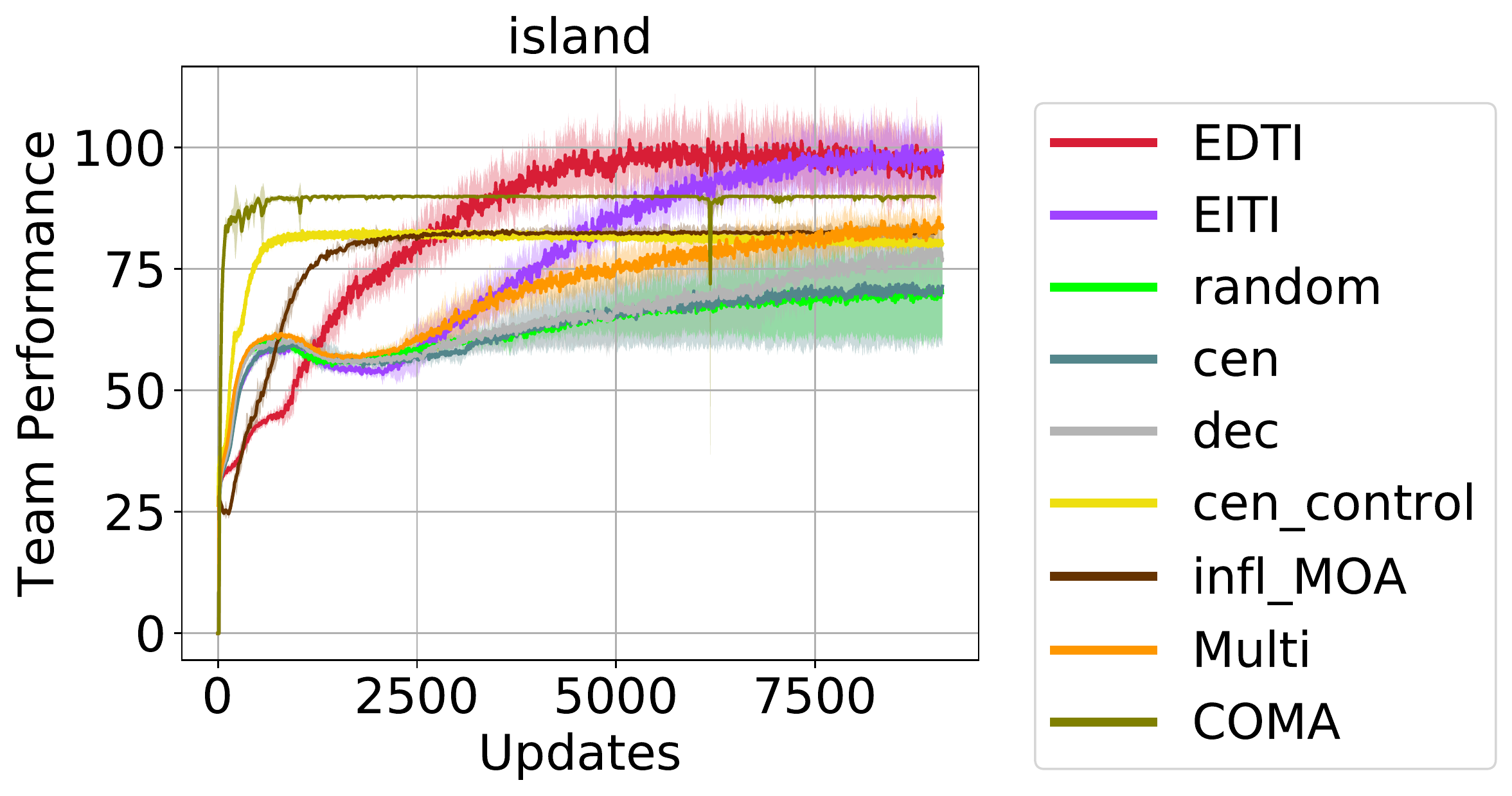}
    \caption{Comparison of our methods against baselines on push-box (left), island (right).}
    \label{fig:push_island_large_island}
\end{figure}

\begin{figure}
    \centering
    \includegraphics[height=0.31\linewidth]{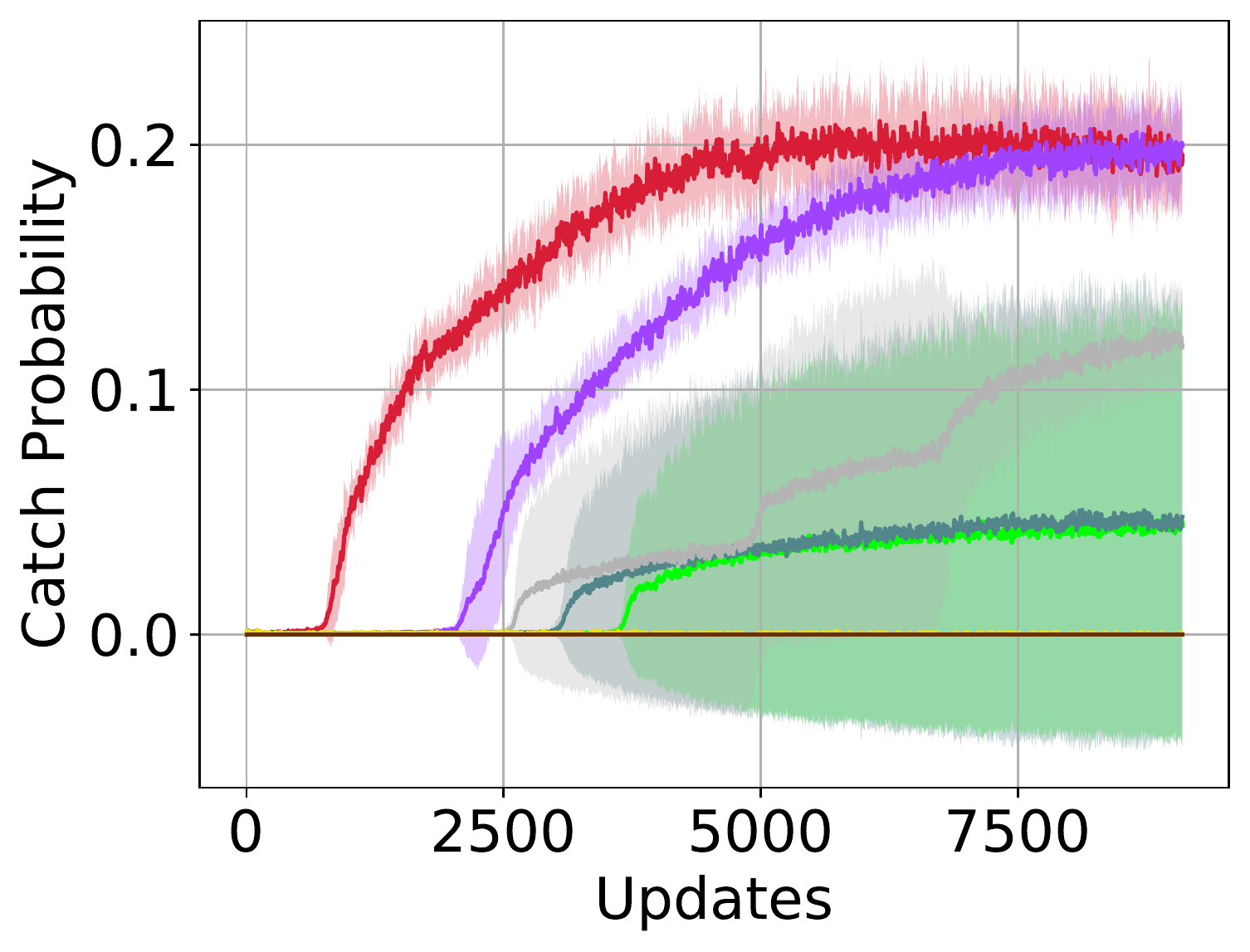}\hfill
    \includegraphics[height=0.31\linewidth]{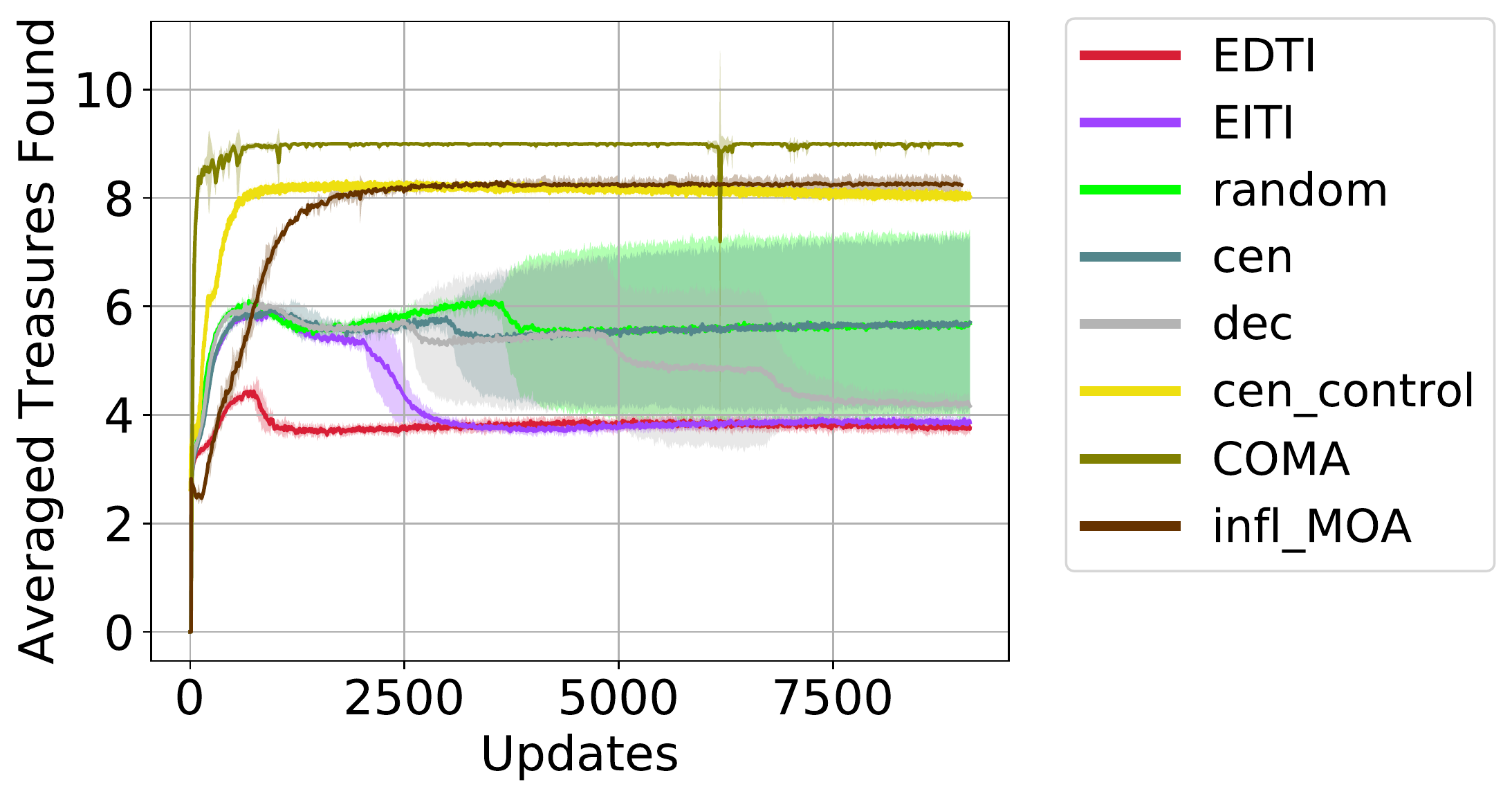} \\
    \includegraphics[height=0.3\linewidth]{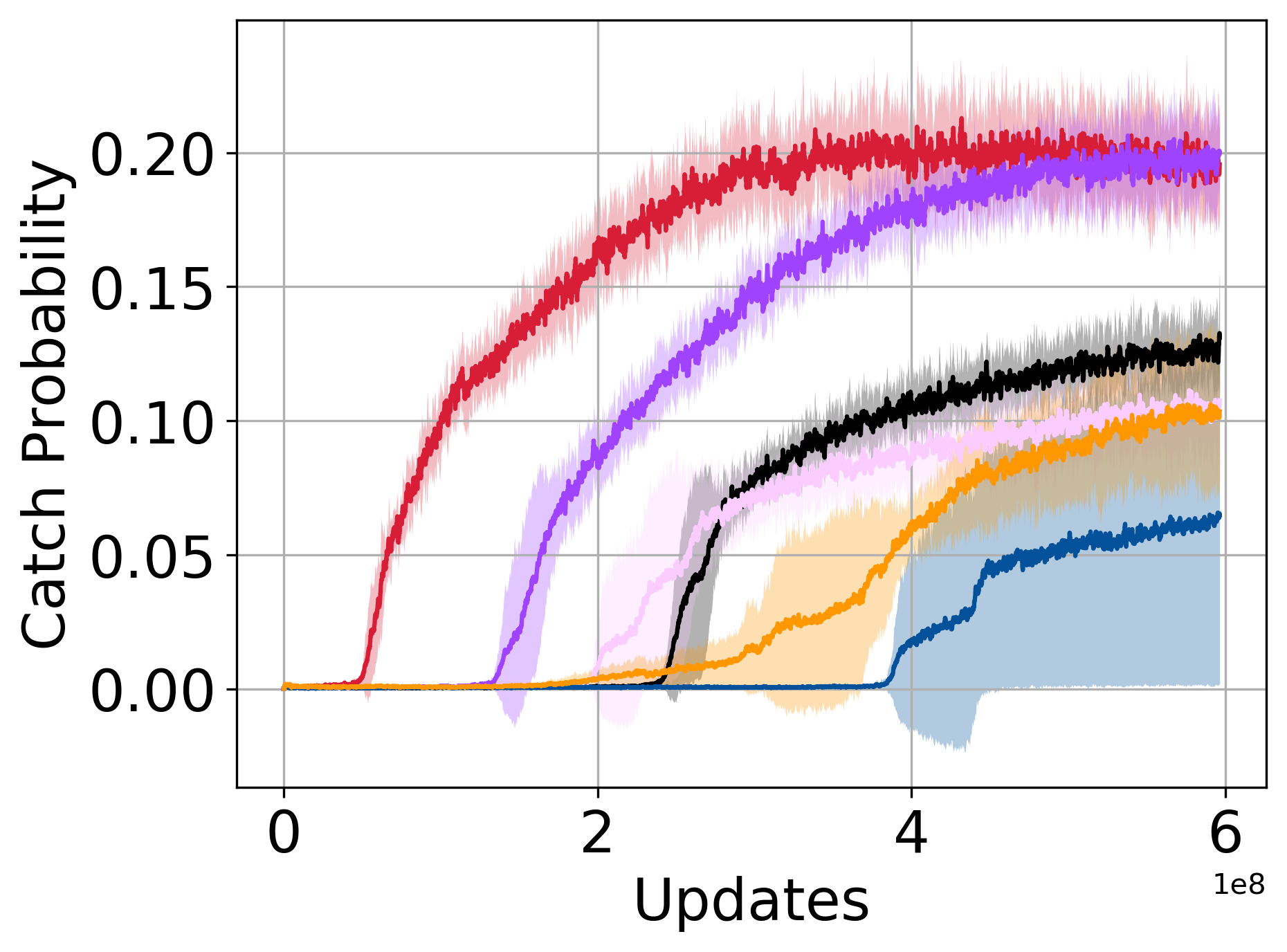}\hfill
    \includegraphics[height=0.3\linewidth]{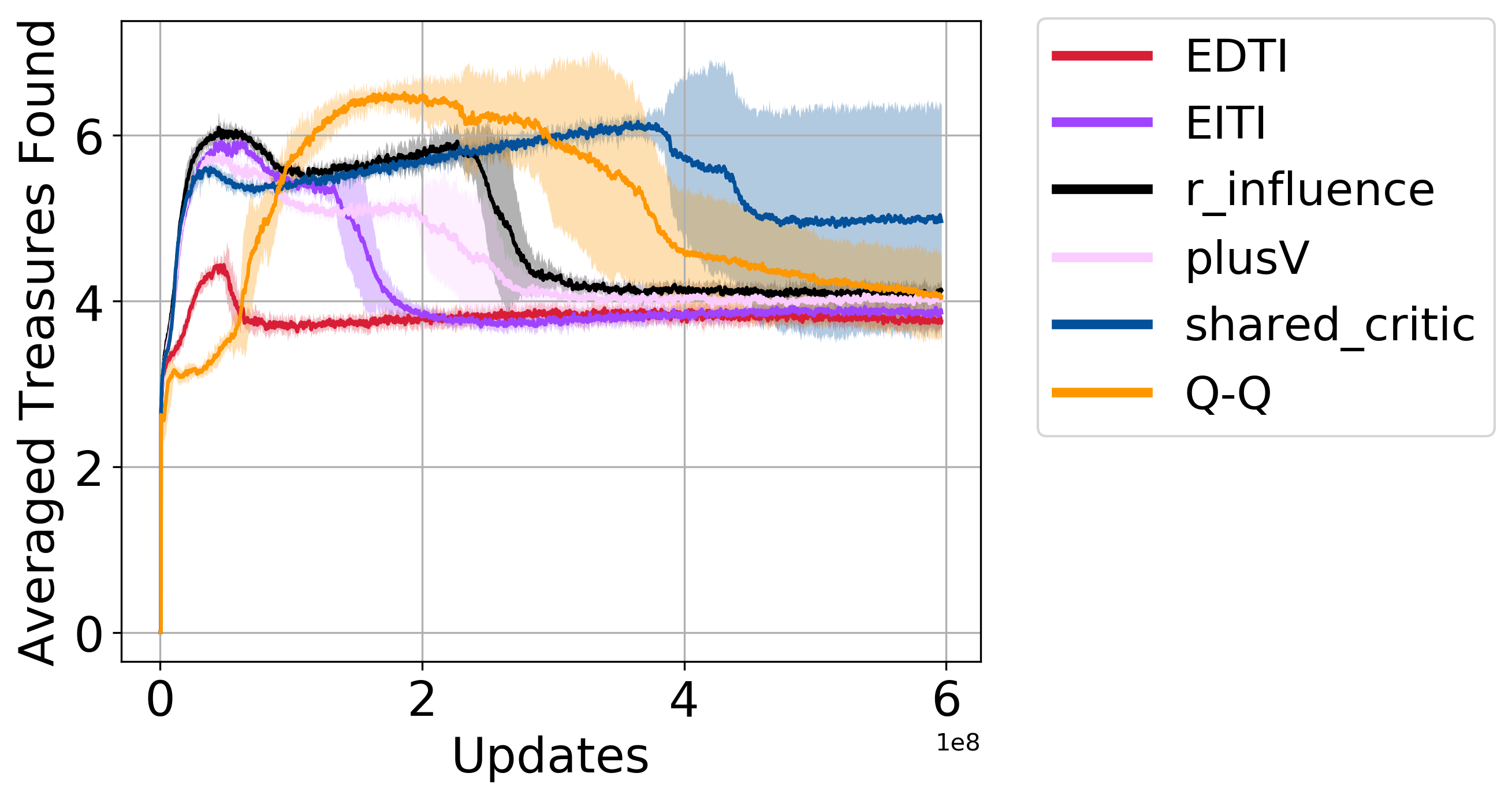} \\
    \caption{Comparison of our methods against baselines and ablations on island in terms of the probability of catching the beast and the averaged treasures collected in an episode.}
    \label{fig:island_details}
\end{figure}

\end{document}